\documentclass{article}

\usepackage{subcaption}

\usepackage{amsmath}
\usepackage{amsfonts}
\usepackage{amssymb}
\usepackage{comment}

\usepackage{graphicx} % more modern

\usepackage{color}

\usepackage{algorithm}
\usepackage{algorithmic}
\usepackage{fullpage}

\newcommand{\RR}{\mathbb{R}}
\newcommand{\NN}{\mathbb{N}}
\newcommand{\sym}{\mathbb{S}}

\makeatletter
\renewcommand\paragraph{\@startsection{paragraph}{4}{\z@}%
  {-3.25ex\@plus -1ex \@minus -.2ex}%
  {1.5ex \@plus .2ex}%
  {\normalfont\normalsize\bfseries}}
\makeatother

\newtheorem{theorem}{Theorem}[section]
\newtheorem{lemma}[theorem]{Lemma}
\newtheorem{proposition}[theorem]{Proposition}
\newtheorem{corollary}[theorem]{Corollary}
\newtheorem{conjecture}[theorem]{Conjecture}
\newtheorem{assumption}[theorem]{Assumption}

\newenvironment{proof}[1][Proof]{\begin{trivlist}
\item[\hskip \labelsep {\bfseries #1}]}{\end{trivlist}}

\newenvironment{remark}[1][Remark]{\begin{trivlist}
\item[\hskip \labelsep {\bfseries #1}]}{\end{trivlist}}

\newcommand{\qed}{\nobreak \ifvmode \relax \else
      \ifdim\lastskip<1.5em \hskip-\lastskip
      \hskip1.5em plus0em minus0.5em \fi \nobreak
      \vrule height0.75em width0.5em depth0.25em\fi}

\newcommand{\norm}[1]{\lVert#1\rVert}

\title{A Comparison of Relaxations of Multiset Cannonical Correlation Analysis and Applications}

\author{Jan~Rupnik,%~\IEEEmembership{Member,~IEEE,}
Primoz~Skraba\footnote{J. Rupnik and P. Skraba are with the Artificial Intelegence Laboratory, Jožef Stefan Institute, Jamova 39, Ljubljana, Slovenia\protect\\
E-mail: jan.rupnik@ijs.si, primoz.skraba@ijs.si} ,%~\IEEEmembership{Member,~IEEE,}
John~Shawe-Taylor\footnote{J. Shawe-Taylor is with the Department of Computer Science, University College London, Gower Street, London WC1E 6BT, United Kingdom\protect\\
E-mail: j.shawe-taylor@ucl.ac.uk},%~\IEEEmembership{Member,~IEEE,}
Sabrina~Guettes\footnote{S. Guettes is with GEN-I, d.o.o.\protect\\
E-mail: sabrina.guettes@gmail.com}% <-this % stops a space%~\IEEEmembership{Member,~IEEE}%
}

\begin{document}
\maketitle

\begin{abstract}
Canonical correlation analysis is a statistical technique that is
used to find relations between two sets of variables. An
important extension in pattern analysis is to consider more than
two sets of variables. This problem can be expressed as a
quadratically constrained quadratic program (QCQP), commonly
referred to Multi-set Canonical Correlation Analysis (MCCA). This
is a non-convex problem and so greedy algorithms converge to
local optima without any guarantees on global optimality. In this
paper, we show that despite being highly structured, finding the
optimal solution is NP-Hard. This motivates our relaxation of the
QCQP to a semidefinite program (SDP). The SDP is convex, can be
solved reasonably efficiently and comes with both absolute and
output-sensitive approximation quality. In addition to theoretical
guarantees, we do an extensive comparison of the QCQP method and
the SDP relaxation on a variety of synthetic and real world
data. Finally, we present two useful extensions: we incorporate
kernel methods and computing multiple sets of canonical vectors.

\end{abstract}

\section{Introduction}
Natural phenomena are often the product of several factors
interacting. A fundamental challenge of pattern analysis and
machine learning is to find the relationships between these
factors. Real world datasets are often modeled using
distributions such as mixtures of Gaussians.  These models often
capture the uncertainty inherent in both underlying systems and
measurements. Canonical correlation analysis (CCA) is a
well-known statistical technique developed to find the
relationships between two sets of random variables.
The relations or patterns discovered by CCA can be used
in two ways. First, they can be used to obtain a common representation for both sets of variables.
% , which reduces the dimensionality of the data and projects the variables to the space of high mutual information.
Second, the patterns themselves can be used in an exploratory analysis (see \cite{Hardoon_usingimage} for example).
%, CCA can be used to relate regions in the visual cortex to visual stimuli (exposure to pleasant versus unpleasant images).
%%This can be thought of as finding the optimal relationship between two underlying factors.
%%This assumes that the even if there are a large number of measurements (i.e. the data lives in a high dimensional space), they are form two groups.

It is possible to extend this idea beyond two sets. The
problem is then known as the Multi-set Canonical Correlation Analysis
(MCCA). Whereas it can be shown that CCA can be solved using an
(generalized) eigenvalue computation, MCCA is a much more
difficult problem. One approach is to express it as a
non-convex quadratically constrained quadratic program (QCQP). In
this paper, we show that despite being a highly structured
problem, it is NP-hard. We then describe an efficient algorithm
for finding a locally optimal solutions to the problem.

Since the algorithm is local and the problem non-convex, we
cannot guarantee the quality of the solutions
obtained. Therefore, we give a relaxation of the problem based on
semi-definite programming (SDP) which gives a constant factor
approximation as well as an output sensitive guarantee.

For use in practical applications, we describe two important
extensions: we adapt the methods to use kernels and to find
multi-dimensional solutions.

Finally, we perform extensive experimentation to compare the
efficient local algorithm and the SDP relaxation on both
synthetic and real-world datasets. Here, we show experimentally
that the hardness of the problem is in some sense generic in low
dimensions. That is, a randomly generated problem in low
dimensions will result in many local maxima which are far from
the global optimum. Somewhat surprisingly, this does not occur in
higher dimensions.%, where convergence to solutions that are not globally optimal is rare.

Our contributions in this paper are as follows:
\begin{itemize}
\item We show that in general MCCA is NP-hard.
\item We describe an scalable and efficient algorithm for finding a locally optimal solution.
\item Using an SDP relaxation of the problem, we can compute a
  global upper bound on the objective function along with various
  approximation guarantees on solutions based on this relaxations.
\item We describe two extensions which are important for practical applications: a kernel method and computing multiple sets of canonical vectors.
\item An extensive experimental evaluation of the respective algorithms: we show that in practice the local algorithm performs extremely well, something we can verify with using the SDP relaxation as well as show there are cases where the local algorithm is far from the optimal solution. We do this with a combination of synthetic and real world examples.
\item We propose a preprocessing step based on random projections, which enables us to apply the SDP bounds on large, high dimensional data sets.
\end{itemize}

The paper is organized as following.
Section \ref{sec:Background} describes the background and related work.
Section \ref{sec:sumcor} introduces the main optimization problem, discusses the problem complexity and presents several bounds on optimal solutions.
% In subsection \ref{subsec:nphard}
% In subsection \ref{subsec:horst}
% In subsection \ref{subsec:globalanalysis}
Section \ref{sec:sumcorextensions} describes the extensions of the original formulation to higher-dimensional, nonlinear case.
% In subsection \ref{subsec:kernels}
% In subsection \ref{subsec:severalCanonicalVectors}
% In subsection \ref{subsec:implementation}
Section \ref{sec:experiments} presents empirical work based on synthetic and real data.
% In subsection \ref{subsec:syndata}
% In subsection \ref{subsec:documents}
Conclusions and future work is discussed in section \ref{sec:discussion}. Finally, in the Appendix we included a primer for the  notation used in the paper.

\section{Background}\label{sec:Background}
Canonical Correlation Analysis (CCA), introduced by Harold Hotelling \cite{Hotelling},  was developed to detect linear relations between two sets of variables. Typical uses of CCA include
statistical tests of dependence between two random vectors, exploratory analysis on multi-view data, dimensionality reduction and finding a common embedding of two random vectors that share mutual information.

 CCA has been generalized in two directions: extending the method to finding nonlinear relations  by using kernel methods \cite{FBMJ}\cite{HardoonCCA} (see \cite{shawe-taylor04kernel} for an introduction to kernel methods) and extending the method to more than two sets of variables which was introduced in \cite{Kettenring}. Among several proposed generalizations in \cite{Kettenring} the most notable is the sum of correlations (SUMCOR) generalization and it is the focus of our paper. There the result is to project $m$ sets of random variables to $m$ univariate random variables, which are pair-wise highly correlated on average\footnote{For $m$ univariate random variables, there are $\binom{m}{2}$ pairs on which we measure correlation}. An iterative method to solve the SUMCOR generalization was proposed in \cite{Horst} and the proof of convergence was established in \cite{Chu}. In \cite{Chu} it was shown that there are exponentially many solutions to a generic SUMCOR problem. In \cite{GlobalMEP} and \cite{GlobalMEP2} some global solution properties were established for special families of SUMCOR problems (nonnegative irreducible quadratic form).
In our paper we show that easily computable necessary and sufficient global optimality conditions are theoretically impossible (which follows from NP-hardness of the problem). Since in
practice good local solutions can be obtained we will present some results on sufficient global optimality.

We also focus on extensions of the local iterative approach \cite{Horst} to make the method practical. Here we show how the method can be extended to finding non-linear patterns and finding more than one set of canonical variates. Our work is related to \cite{JointBSSAppl} where a deflation scheme is used together with the Newton method to find several sets of canonical variates. Our nonlinear generalization is related to \cite{nonlinJointBSS}, where the main difference lies in the fact that we kernelized the problem, whereas the authors in \cite{nonlinJointBSS} worked with explicit nonlinear feature representation.

We now list some applications of the SUMCOR formulation. In \cite{kernelHyperAppl} an optimization problem for multi-subject functional magnetic resonance imaging (fMRI) alignment is proposed, which can be formulated as a SUMCOR problem (performing whitening on each set of variables). Another application of the SUMCOR formulation can be found in \cite{JointBSSAppl}, where it is used for group blind source separation on fMRI data from multiple subjects. An optimization problem equivalent to SUMCOR also arises in control theory \cite{ControlApplication} in the form of linear sensitivity analysis of systems of differential equations.

\section{Sum of correlations optimization problem}\label{sec:sumcor}%prej \label{eq:qcqp}

Before stating the problem, we must introduce some notation and
context.  Assume that we have a random vector $\mathcal{X}$
distributed over $\RR^N$, which is centered: $E\left(\mathcal{X}\right) = 0$. Let
$C := Cov\left(\mathcal{X}, \mathcal{X}\right)$ denote the covariance matrix of $\mathcal{X}$.

Throughout the paper we will use the block matrix and vector notation.
The \emph{block structure} $$b := \left(n_1, \ldots, n_m\right), \sum_i b\left(i\right) = N$$ denotes the number
of elements in each of $m$ blocks. Sub-vectors according to the block structure $b$ are denoted as $\mathcal{X}^{(i)} \in \RR^{n_i}$
($i$-th block-row of vector $\mathcal{X}$) and sub-matrices as $C^{(i,j)} \in \RR^{n_i \times n_j}$ ($i$-th block-row, $j$-th block column of matrix $C$); see Figure~\ref{fig:block_structure}.

\begin{figure}[t]
\centering
\includegraphics[width=0.5\textwidth]{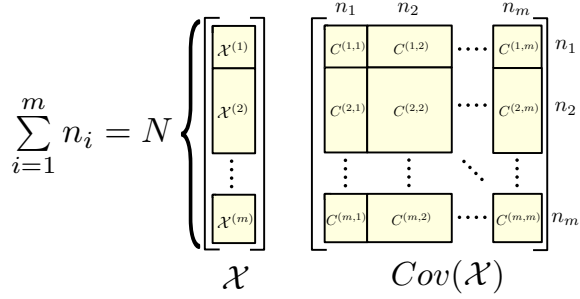}
\caption{\label{fig:block_structure} The block structure of the  random vector $\mathcal{X}$ and the corresponding covariance block structure.}
\end{figure}
For a vector, $w \in \RR^N$, define
\begin{equation*}
\mathcal{Z}_i := \sum_{j = 1}^{n_i} \mathcal{X}^{(i)}\left(j\right)
w^{(i)}\left(j\right) = \mathcal{X}^{(i)T} \cdot w^{(i)},
\end{equation*}
$\mathcal{Z}_i$ is a random variable computed as linear combination of components of
$\mathcal{X}^{(i)}$.

Let $\rho\left(x,y\right)$ denote the correlation
 coefficient between two random variables,
\begin{equation*}
\rho\left(x,y\right) =
 \frac{Cov\left(x,y\right)}{\sqrt{Cov\left(x,x\right) Cov\left(y,y\right)}}
\end{equation*}
 The correlation coefficient between $\mathcal{Z}_i$ and
 $\mathcal{Z}_j$ can be expressed as:
\begin{equation*}
\rho\left(\mathcal{Z}_i, \mathcal{Z}_j\right) = \frac{w^{(i)T} C^{(i,j)}
   w^{(j)}}{\sqrt{w^{(i)T} C^{(i,i)} w^{(i)}}\sqrt{w^{(j)T}
     C^{(j,j)} w^{(j)}} }
\end{equation*}
We are now ready to state the problem. In this paper, we deal
with the problem of finding an optimal set of vectors $w^{(i)}$
which maximize
\begin{equation}\label{eq:SUMCOR}
\tag{SUMCOR}
\sum_{i = 1}^m \sum_{j = i+1}^m
\rho\left(\mathcal{Z}_i, \mathcal{Z}_j\right)
\end{equation}
This is a generalization of Canonical Correlation Analysis where
$m=2$. We refer to this problem as Multi-set Canonical
Correlation Analysis (MCCA).  We refer to each
$\mathcal{X}^{(i)}$ as a particular view of some underlying
object with the assumption is that random vectors
$\mathcal{X}^{(i)}$ share some mutual information (i.e. are not independent).
The original sum of correlations optimization problem is:
\begin{equation*}
\begin{aligned}
& \underset{w \in \RR^N}{\text{max}} & & \sum_{i = 1}^m
  \sum_{j = i+1}^m \frac{w^{(i)T} C^{(i,j)}
    w^{(j)}}{\sqrt{w^{(i)T} C^{(i,i)} w^{(i)}} \sqrt{w^{(j)T}
      C^{(j,j)} w^{(j)}}}.
\end{aligned}
\end{equation*}
%
%%some explanation
%
The solution to the optimization problem, the set of components, $\left(w^{(1)}, \ldots, w^{(m)}\right),$ are referred to as the set of canonical vectors.
Observe that the solution is invariant to scaling (only the direction matters): if $\left(w^{(1)}, \ldots, w^{(m)}\right)$ is a solution, then $\left(\alpha_1 \cdot w^{(1)}, \ldots, \alpha_m \cdot w^{(m)}\right)$ is also a solution for $\alpha_i > 0$. This means that we have the freedom to impose the constraints $w^{(i)T}C^{(i,i)}w^{(i)} = 1$, which only affect the norm of the solutions. We now
arrive to an equivalent constrained problem:
\begin{equation}\label{eq:qcqp0}
\begin{aligned}
& \underset{w \in \RR^N}{\text{maximize}}
& & \sum_{i = 1}^m \sum_{j = i+1}^m w^{(i)T} C^{(i,j)} w^{(j)}\\
& \text{subject to}
& &w^{(i)T} C^{(i,i)} w^{(i)} = 1, \quad\forall i = 1,\ldots, m.
\end{aligned}
\end{equation}
We proceed by multiplying the objective by $2$ and adding a constant $m$, which does not affect the optimal solution. Using the equalities: $w^{(i)T} C^{(i,j)} w^{(j)} = w^{(j)T} C^{(j,i)} w^{(i)}$ and $w^{(i)T} C^{(i,i)} w^{(i)} = 1$ we arrive at:
\begin{equation}\label{eq:qcqp05}
\begin{aligned}
& \underset{w \in \RR^N}{\text{maximize}}
& & \sum_{i = 1}^m \sum_{j = 1}^m w^{(i)T} C^{(i,j)} w^{(j)}\\
& \text{subject to}
& &w^{(i)T} C^{(i,i)} w^{(i)} = 1, \quad\forall i = 1,\ldots, m.
\end{aligned}
\end{equation}
This allows us to consider the summation as the quadratic form $w^T C w$.

Let $C^{(i,i)}$ is strictly positive definite, the Cholesky
decomposition $C^{(i,i)} = D_i^T D_i$ exists. Using the substitution
$\widetilde{w}_i := D_i w_i$  and defining $A \in \RR^N$ such that
\begin{equation*}
A^{(i,j)} := {D_i}^{-T} C^{(i,j)} {D_j}^{-1}
\end{equation*}
Here the block structure $b$ is used. As a consequence of the
substitution, we have that $A^{(i,i)} = I_{b\left(i\right)}$, where $I$
denotes the $b\left(i\right)$-by-$b\left(i\right)$ identity matrix. Using block
vector notation, let $x \in \RR^N, x^{(i)} = \widetilde{w}_i,
\forall i = 1,\ldots, m$.

%Let $b := (n_1,\ldots,n_m)$ encode a block structure, $N := \sum_i n_i$ and $A \in \sym_N^{+}$.
The form of the optimization problem we will finally consider is:
\begin{equation}\label{eq:qcqp}
\tag{QCQP}
\begin{aligned}
& \underset{x \in \RR^{N}}{\text{max}}
& & x^T A x\\
& \text{subject to}
& &x^{(i)T} x^{(i)} = 1, \quad\forall i = 1, \ldots, m,.
\end{aligned}
\end{equation}
where  $b := \left(n_1,\ldots,n_m\right)$ encodes the block structure,  $A \in \sym_N^{+}$, $A^{(i,i)} = I_{b\left(i\right)}$.

We started with a formulation (\ref{eq:SUMCOR}) and arrived to (\ref{eq:qcqp}). The last optimization problem is
simpler to manipulate and will be used to prove the complexity result of (\ref{eq:SUMCOR}), as well as to obtain
a relaxed version of the problem along with some useful bounds. We will also state a local-optimization approach to solving (\ref{eq:SUMCOR})
based on the (\ref{eq:qcqp}) problem formulation.
%\textcolor{red}{Jan: moramo mal vec prikazat - vec razlozit zakaj smo tko naredil...}

\subsection{NP-Hardness}\label{subsec:nphard}
First, we give a reduction to show that this optimization problem
is NP-hard. We use a reduction from a general binary quadratic
optimization (BQO) problem.

Let $A   \in \RR^{m\times m}$  %and $x := (x_1, \ldots, x_m)^T \in \RR^m$.
the binary quadratic optimization (BQO) problem is
\begin{equation}\label{eq:binqp}
\tag{BQO}
\begin{aligned}
& \underset{x \in \RR^m}{\text{max}}
& & x^T A x  \\
& \text{subject to}
& & x\left(i\right)^2 = 1,  \quad\forall i =1,\ldots,m\\
\end{aligned}
\end{equation}

Many difficult combinatorial optimization problems (for example
maximum cut  problem and maximum clique problem~\cite{Garey:1990:CIG:574848}) can be reduced to
BQO \cite{Goemans95improvedapproximation}, which is known to be NP-hard.

%\textcolor{red}{a je ocitno da to nice ne spremeni }

We will show that each BQO problem instance can be reduced to an
instance of the problem (\ref{eq:qcqp}). That means that even though the
problem (\ref{eq:qcqp}) has special structure (maximizing a positive-definite quadratic form over a
product of spheres) it still falls into the class of problems that are hard (under the assumption that $P \neq NP$).
The idea is to start with a general instance of BQO and through a set of simple transformations obtain a specific
instance of (\ref{eq:qcqp}), with a block structure $b = \left(1,\ldots,1\right)$. The simple transformations will transform the
BQO matrix $A$ into a correlation matrix, where the optimal solutions will be preserved.

  Let us start a BQO with a corresponding general matrix $A \in \RR^{m\times m}$. Since $x^T A x = x^T
\frac{\left(A + A^T\right)}{2} x$ we can assume that the matrix $A$ is
symmetric. The binary constraints imply that for any diagonal
matrix $D$ the quantity $x^T D x = \sum_i D\left(i,i\right)$ is
constant. This means that for $c > 0$ large enough, we can
replace the objective with an equivalent objective $x^T \left(A + c
\cdot I\right) x$ which is a positive-definite quadratic form.  If we
set $c$ to $\norm{A}_1 + 1$, it guarantees strong diagonal
dominance, a sufficient condition for positive definiteness. From
now on, we assume that the matrix $A$ in the BQO is symmetric and
positive-definite.  Let $g = \max_i{A\left(i,i\right)}$ and let $D \in
\RR^{m\times m}$ be a diagonal matrix with elements $D\left(i,i\right) = g -
A\left(i,i\right)$.

Then the BQO is equivalent to:
\begin{equation}\label{eq:binqpcor}
\begin{aligned}
& \underset{x \in \RR^m}{\text{max}}
& & x^T \frac{\left(A + D\right)}{g} x  \\
& \text{subject to}
& & x\left(i\right)^2 = 1,  \quad\forall i =1,\ldots,m.\\
\end{aligned}
\end{equation}
The matrix $\frac{\left(A + D\right)}{g}$ is a correlation matrix since it
is a symmetric positive-definite matrix with all diagonal entries
equal to $1$. The optimization problem corresponds to a problem
of maximizing a sum of pairwise correlations between univariate
random variables (using block structure notation: $b\left(i\right) = 1,
\forall i = 1,\ldots, m$). This shows that even the simple case
of maximizing the sum of correlations, where the optimal axes are
known and only directions need to be determined, is a NP-hard
problem.

\subsection{Local solutions}\label{subsec:horst}
Given that the problem is NP-hard and assuming $P\neq NP$, it is
natural to use local methods to obtain a perhaps suboptimal
solution.  In this section, we give an algorithm that
provably converges to locally optimal solutions of the problem
(\ref{eq:qcqp}), when the involved matrix $A$ is symmetric and
positive-definite and generic (see \cite{Chu}).

The general iterative procedure is given as Algorithm \ref{algorithm:horst}.
\begin{algorithm}
\caption{Horst algorithm}
\label{algorithm:horst}
{\bf Input:} matrix $A \in \sym_N^+$, block structure $b = \left(n_1,\ldots,n_m\right)$, initial vector $x_0 \in \RR^N$ with $\norm{x^{(i)}} > 0$,  \par
\begin{algorithmic}
\STATE $x \leftarrow x_0;$
\FOR{$iter = 1$ to $maxiter$}
\STATE $x \leftarrow A x;$
\FOR{$i =1$ to $m$}
\STATE $x^{(i)} \leftarrow \frac{x^{(i)}}{\norm{x^{(i)}}}$
\ENDFOR
\ENDFOR
\end{algorithmic}
{\bf Output:} $x$
\end{algorithm}
%\STATE $\alpha_j^{i} \leftarrow  \sum_{k}A_{j,k } \alpha_k^{i-1}$
%\STATE $\alpha_j^i \leftarrow \frac{\alpha_j^{i}}{\sqrt{{\alpha_j^i}' \alpha_j^i}}$
% Every step of the main loop in the algorithm \ref{algorithm:horst} involves $m^2$ matrix vector multiplications.
% We can exploit the structure of the problem, namely the fact that the MEP matrix is a sum of block diagonal matrices and a low rank block matrix (with block column matrix as its Cholesky factor) and gain a speed up by a factor of $m$.

The algorithm can be interpreted as a generalization of the power iteration method (also known as the Von Mises iteration), a classical approach to finding the largest solutions to the eigenvalue problem $A x = \lambda x$. If the number of views $m=1$, then Algorithm \ref{algorithm:horst} exactly corresponds to the power iteration. Although the algorithm's convergence is guaranteed under the assumptions mentioned above, the convergence rate is not known. In practice we observe linear convergence, which we demonstrate in figures \ref{fig:qcqpConvVar}, \ref{fig:qcqpConvFix}. In figure \ref{fig:qcqpConvVar} we generated $1000$ random\footnote{We used random Gram matrix method to generate random problem instances; the method is described in section \ref{subsec:syndata}} instances of matrices $A$ with block structure $b = \left(2,2,2,2,2\right)$ and for each matrix we generated a starting point $x_0$ and ran the algorithm. The figure depicts the solution change rate on a logarithmic scale ($log_{10} \frac{\norm{x_{old} - x}}{\norm{x}}$). We observe linear convergence with a wide range of rates of convergence (slopes of the lines). Figure \ref{fig:qcqpConvFix} shows the convergence properties for a fixed matrix $A$ with several random initial vectors $x_0$. The problem exhibits a global (reached for $65\%$ initial vectors $x_0$) and a local solution (reached in $35\%$ cases), where the global solution paths converge faster (average global solution path slope equals $-0.08$, as opposed to $-0.05$ for the local solution paths).

\begin{figure*}[htbp]
  \centering
\begin{subfigure}[b]{0.4\textwidth}
    \includegraphics[width=\textwidth]{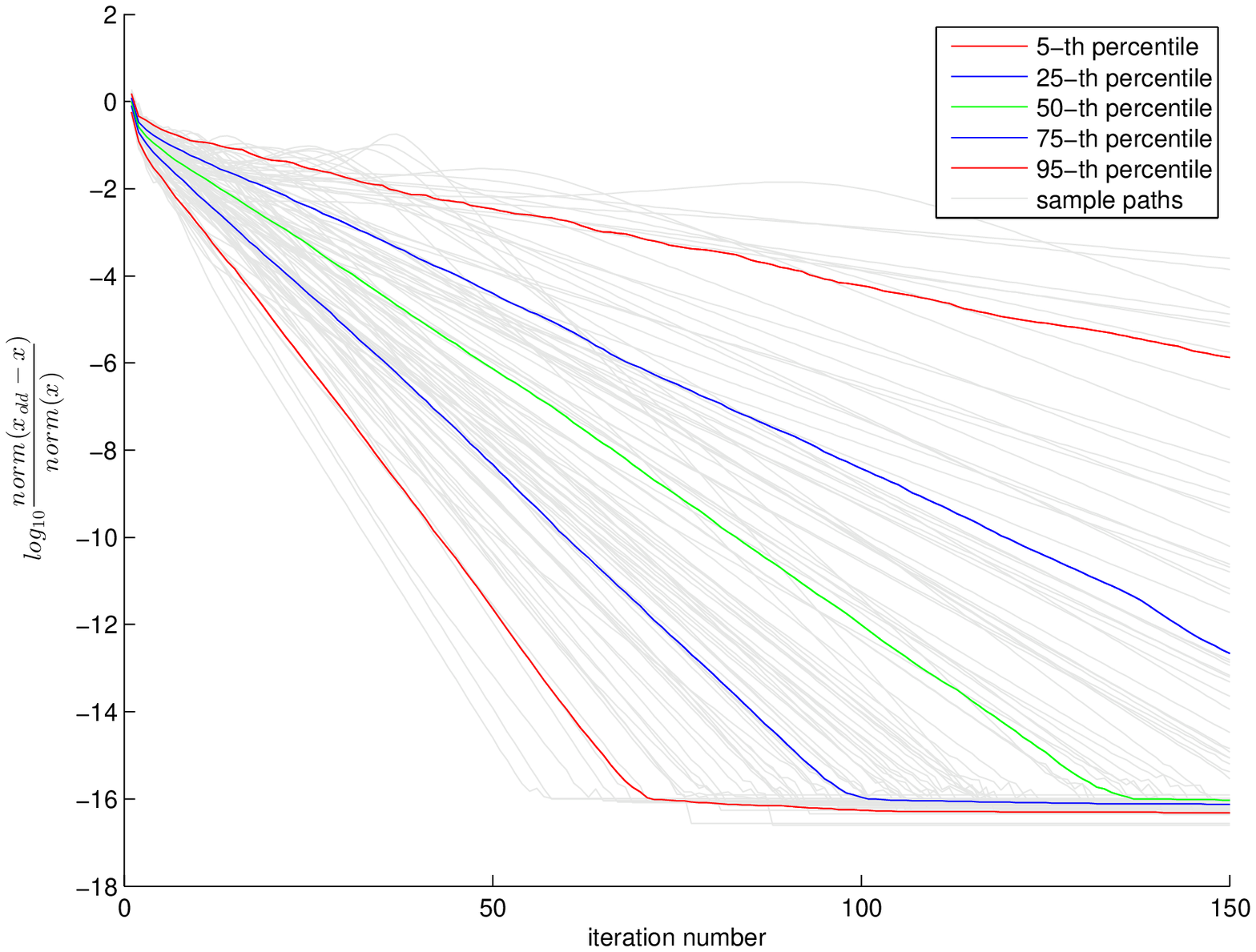}
\caption{\label{fig:qcqpConvVar}}
\end{subfigure}\hspace{0.1\textwidth}
\begin{subfigure}[b]{0.4\textwidth}
    \includegraphics[width=\textwidth]{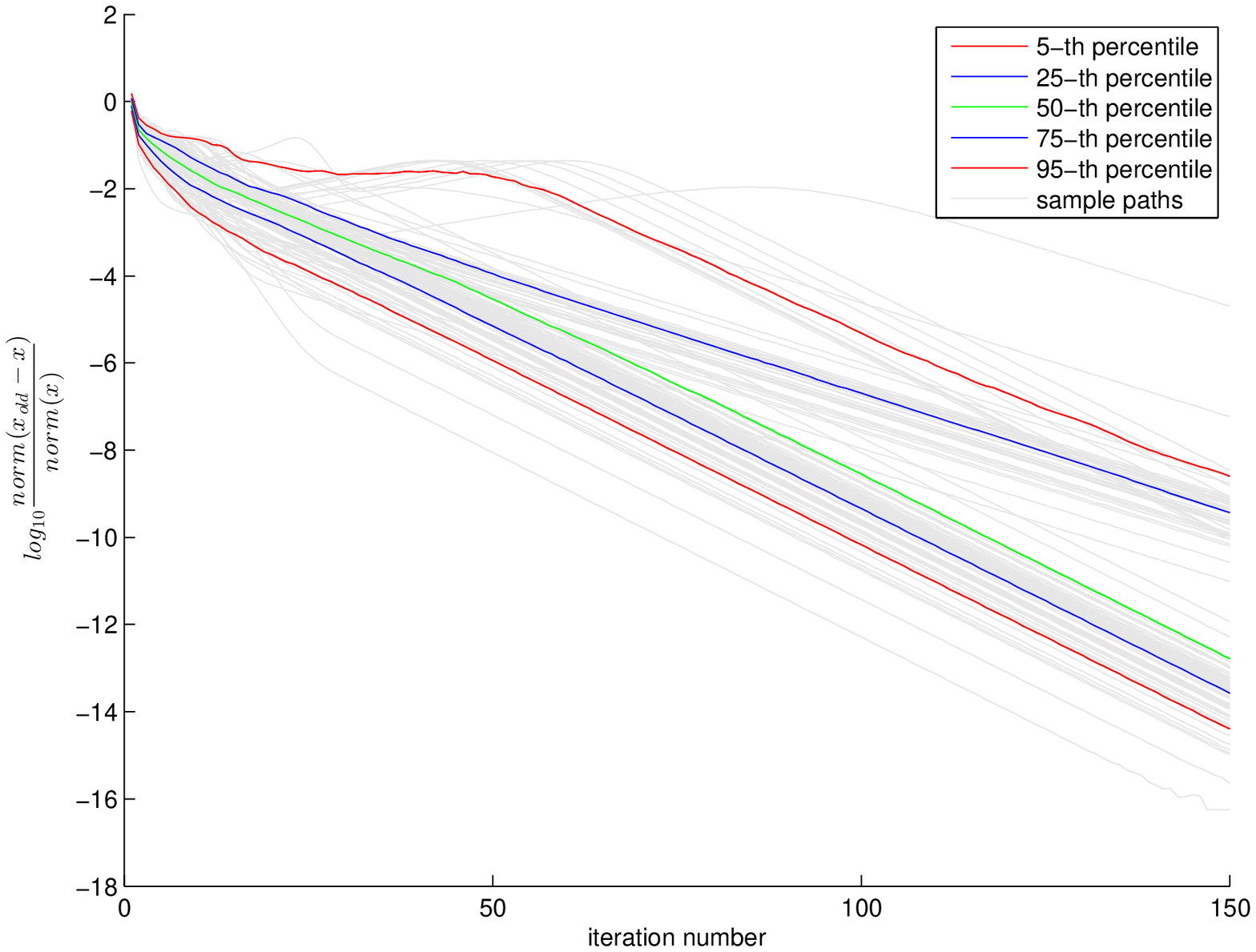}
\caption{\label{fig:qcqpConvFix}}
\end{subfigure}
  \caption{(a) Convergence plot ($1000$ random matrices $A$, one random $x_0$ per problem instance) (b) Convergence plot (single random $A$, $1000$ random initial vectors $x_0$)}
\end{figure*}

\subsection{Global analysis}\label{subsec:globalanalysis}
The above algorithm is highly scalable and as we shall see in
Section~\ref{sec:experiments} often works well in practice.
However the algorithm may not converge to a globally optimal
solution.  We will first present the semidefinite programming
relaxation of the problem in \ref{subsubsec:sdp}. There we will
show how the SDP solution can be used to extract solution
candidates for the original problem. We will prove a bound that
relates the extracted solution quality and the optimal QCQP
objective value.

We will then present a set of upper bounds on the optimal QCQP objective value \ref{subsubsec:upperbounds}.
Such bounds can be used as certificates of optimality (or closeness to optimality) of local solutions (obtained by the local iterative approach for example).

\subsubsection{Semidefinite programming relaxation}\label{subsubsec:sdp}

\noindent\textbf{SDP Relaxation}
Let $A, B_1, \ldots, B_m  \in \RR^{N \times N}$ share the block structure $b := \left(n_1, \ldots, n_m\right), \sum_i b\left(i\right) = N$.
The blocks $B_i^{(k,l)} \in \RR^{n_k \times n_l}$ for $i,k,l = 1,\ldots,m$ are defined as:
$$
B_i^{(k,l)} := \left\{
     \begin{array}{l l}
       I_{n_i} & : k = i, l = i\\
       0_{k,l} & : else
     \end{array}
   \right.
$$
where $I_{n_i} \in R^{n_i \times n_i}$ is an identity matrix and $0_{k,l} \in \RR^{k,l}$ is a matrix with all entries equal to zero.

Using the fact that $x^T A x = \mathrm{trace}\left( A x x^T\right)$ we can rewrite the problem (\ref{eq:qcqp}) as:
\begin{equation}\label{eq:qcqp2}
\tag{QCQP2}
\begin{aligned}
& \underset{x \in \RR^n}{\text{max}}
& &\mathrm{trace}\left(A x x^T\right) \\
& \text{subject to}
& &\mathrm{trace}\left(B_i x x^T\right) = 1,  \quad\forall i =1,\ldots,m\\
\end{aligned}
\end{equation}
We can substitute the matrix $x x^T$ with a general matrix $X \in \sym_{+}^{n}$ constrained it to being rank-one:
\begin{equation*}
\begin{aligned}
& \underset{X \in \sym_{+}^n}{\text{maximize}}
& &\mathrm{trace}\left(A X\right) \\
& \text{subject to}
& &\mathrm{trace}\left(B_i X\right) = 1,  \quad\forall i =1,\ldots,m\\
& & &\mathrm{rank}\left(X\right) = 1.
\end{aligned}
\end{equation*}
\begin{remark}
Matrices $A$ and $B_1,\ldots,B_m$ are symmetric positive-semidefinite matrices.
%This is obvious for matrices $B_i$ and observe that matrix $A$ is similar to a covariance matrix corresponding to random vector $(X_1',\ldots, X_m')'$, hence positive-semidefinite.
\end{remark}
By omitting the rank-one constraint we obtain a semi-definite program in standard form:
\begin{equation}\label{eq:sdp}
\tag{SDP}
\begin{aligned}
& \underset{X \in \sym_{+}^n}{\text{max}}
& &\mathrm{trace}\left(A X\right) \\
& \text{subject to}
& &\mathrm{trace}\left(B_i X\right) = 1,  \quad\forall i =1,\ldots,m.\\
\end{aligned}
\end{equation}
\begin{remark}
If the solution of the problem (\ref{eq:sdp}) is rank-one, i.e. $X$ can be expressed as $X = y \cdot y^T$, then $y$ is the optimal solution for (\ref{eq:qcqp}).
 %for $y = (y_1^T, \ldots, y_m^T)^T\in \RR^{N}$, where $y_i \in \RR^{n_i}$ are block components of vector $y$, then a solution to the problem (\ref{eq:qcqp}) is obtained by setting $x_i = y_i$.
\end{remark}
\noindent\textbf{Low rank solutions}
In the following subsection we will show how to extract solutions to QCQP from solutions of the SDP problem. We will present a bound that relates
the global SDP bound, the quality of the extracted solution and the optimal value of QCQP. The bound will tell us how the extracted solution gets close to the optimal QCQP
solution when the SDP solution is close to rank 1.

Let $X^*$ be a solution to the problem (\ref{eq:sdp}) and let $x^{*}$ be the solution to the problem (\ref{eq:qcqp}). Then the following inequality always holds:
 $$\mathrm{trace}\left(A X^{*}\right) \geq \mathrm{trace}\left(A \cdot x^{*} \cdot x^{*T}\right).$$

An easy way to extract a good feasible solution to the problem (\ref{eq:qcqp}) from $X^*$ is to project its leading eigenvector to the set of constraints. Let $b = \left(n_1,\ldots,n_m\right), \sum_i n_i = N $ denote the block structure.

Let $y \in \RR^N, \norm{y^{(i)}} \neq 0$. The projection of vector $y$ to the feasible set of the problem (\ref{eq:qcqp}) is given by map $\pi\left(\cdot\right) : \RR^N \rightarrow \RR^N$, defined as:
$$\pi\left(y\right) := \left(\frac{y^{(1)}}{\norm{y^{(1)}}}, \ldots, \frac{y^{(m)}}{\norm{y^{(m)}}}\right).$$
The quality of the solution depends on spectral properties of matrix $X$ and matrix $A$.

\begin{assumption}\label{thm:assumption1}
Let $b = \left(n_1,\ldots,n_m\right)$ denote the block structure and $ \sum_i n_i = N $.
Let $X^*$ be the solution to the problem (\ref{eq:sdp}). Let $x_k$ denote the $k$-th eigenvector of $X^*$.
The assumption is the following: $$\norm{x_1^{(i)}} > 0, \forall i = 1,\ldots, m.$$
\end{assumption}

\begin{conjecture}\label{thm:conj1}
Assumption \ref{thm:assumption1} holds in general for optimal solutions to  the problem (\ref{eq:sdp}).
\end{conjecture}
The assumption makes the projection to the feasible set,
$\pi(\cdot)$, well defined. In our experiments, this was always
true, but we have been unable to find a proof. The result that we
will state after the next lemma is based on the projection
operator and thus relies on the assumption.
\begin{lemma}
Let $b = \left(n_1,\ldots,n_m\right)$ denote the block structure and $ \sum_i n_i = N $.
Let $X^*$ be the solution to problem (\ref{eq:sdp}).
Let $x_k$ denote the $k$-th eigenvector of $X^*$.
Let $\alpha_i := \frac{1}{\norm{x_1^{(i)}}}$.
If $X^*$ can be expressed as:
$$X^* = \lambda_1  x_1 x_1^T + \lambda_2 x_2 x_2^T,$$
where $x_1$ and $x_2$ have unit length and $\lambda_1 > 1 >  \lambda_2$, then
$$\lambda_1 \leq \alpha_i \alpha_j  \leq \frac{\lambda_1}{1 - \lambda_2}.$$
\end{lemma}
\begin{proof}
The constraints in problem (\ref{eq:sdp}) are equivalent to:
$$\lambda_1 \norm{x_1^{(i)}}^2 + \lambda_2 \norm{x_2^{(i)}}^2 = 1, \forall i = 1,\ldots, m.$$
Since $\lambda_2 < 1$ and $\norm{x_2^{(i)}}^2 \leq 1$ it follows that $0 \leq \lambda_2 \norm{x_2^{(i)}}^2  < 1$.
It follows that:
$$  0 < \frac{1- \lambda_2}{\lambda_1} < \norm{x_2^{(i)}}^2 \leq \frac{1}{\lambda_1}.$$
Since $\alpha_i = \frac{1}{x_1^{(i)}}$ it follows that $$ \sqrt{\lambda_1} \leq \alpha_i \leq \sqrt{\frac{\lambda_1}{1-\lambda_2}},$$
and finally:
$$ \lambda_1 \leq \alpha_i \alpha_j \leq \frac{\lambda_1}{1-\lambda_2}, \forall i,j = 1,\ldots,m.$$
\qed
\end{proof}
\begin{proposition}
Let $b = \left(n_1,\ldots,n_m\right)$ denote the block structure and $ \sum_i n_i = N $.
Let $X^*$ be the solution to the problem (\ref{eq:sdp}) and $x^{*}$ be the solution to the problem (\ref{eq:qcqp}).
Let $x_k$ denote the $k$-th eigenvector of $X^*$.
Let $\alpha_i := \frac{1}{\norm{x_1^{(i)}}}$.
Let $\psi := \mathrm{trace}\left(A X^{*}\right)$, $\phi := \mathrm{trace}\left(A \cdot x^{*} \cdot x^{*T}\right)$.
%$\pi(OP_3) := \mathrm{trace}(A \cdot \pi(x_1) \cdot \pi(x_1^T)$.
It is obvious that: $$\psi \geq \phi \geq \pi\left(x_1\right).$$
If $X^*$ can be expressed as:
$$X^* = \lambda_1  x_1 x_1^T + \lambda_2 x_2 x_2^T,$$
where $x_1$ and $x_2$ have unit length and $\lambda_1 > 1 >  \lambda_2$,
 then: $$\psi - \pi\left(x_1\right) \leq  \left(\frac{1}{1-\lambda_2} -1\right)  m^2 + \lambda_2 \norm{A}_2.$$
\end{proposition}
\begin{proof}
\begin{align*}
\psi &- \pi\left(x_1\right) = \lambda_1 \sum_{i,j} x_1^{(i)T} A^{(i,j)} x_1^{(j)T}  + \\
&\lambda_2\sum_{i,j} x_2^{(i)T} A^{(i,j)} x_2^{(j)T} - \sum_{i,j} \alpha_i \alpha_j x_1^{(i)T} A^{(i,j)} x_1^{(j)T}  \leq \\
&\leq  \sum_{i,j} (\lambda_1 - \alpha_i \alpha_j) x_1^{(i)T} A^{(i,j)} x_1^{(j)T}  + \lambda_2 \norm{A}_2 \leq\\
&\leq -\sum_{i,j} (\lambda_1 - \alpha_i \alpha_j)\cdot |x_1^{(i)T} A^{(i,j)} x_1^{(j)T}| +  \lambda_2 \norm{A}_2 \leq\\
&\leq \left(\frac{\lambda_1}{1-\lambda_2} - \lambda_1 \right)\cdot \sum_{i,j}  \frac{1}{\alpha_i \alpha_j} +  \lambda_2 \norm{A}_2 \leq\\
&\leq \left(\frac{\lambda_1}{1-\lambda_2} - \lambda_1 \right)\cdot \frac{m^2}{\lambda_1} +  \lambda_2 \norm{A}_2 =\\
&= \left(\frac{1}{1-\lambda_2} - 1 \right)\cdot m^2 +  \lambda_2 \norm{A}_2
\end{align*}
\qed
\end{proof}
A similar bound can be derived for the general case, provided that the solution to problem (\ref{eq:sdp}) is close to rank-one.
\begin{proposition}
Let $b = \left(n_1,\ldots,n_m\right)$ denote the block structure and $ \sum_i n_i = N $.
Let $X^*$ be the solution to the problem (\ref{eq:sdp}) and $x^{*}$ be the solution to the problem (\ref{eq:qcqp}).
Let $x_k$ denote the $k$-th eigenvector of $X^*$.
Let $\alpha_i := \frac{1}{\norm{x_1^{(i)}}}$.
Let $\psi := \mathrm{trace}\left(A X^{*}\right)$. %, $\phi := \mathrm{trace}(A \cdot x^{*} \cdot x^{*T})$.
%Let $a_i := \frac{1}{\norm{x^{(i)}_1}}$ and let $OP_3 := \mathrm{trace}(A X^{*})$, $OP_2 := \mathrm{trace}(A \cdot x^{**} \cdot x^{**'})$ and
%$\pi(OP_3) := \mathrm{trace}(A \cdot \pi(x^{(1)}) \cdot \pi(x^{(1)})')$.
If $X^*$ can be expressed as:
$$X^* = \lambda_1  x_1 x_1^T + \lambda_2 x_2 x_2^T + \cdots + \lambda_n x_n x_n^T,$$
where each $x_i$ has unit length and $\lambda_1 > 1 >  \underset{i=2,\ldots, n}{\sum} \lambda_i$,
 then:
\begin{align*}
\psi &- \pi\left(x_1\right) \\ &\leq  \left(\frac{1}{1- \underset{i=2,\ldots, n}{\sum} \lambda_i} -1\right)  m^2 + \left(\underset{i=2,\ldots, n}{\sum} \lambda_i\right) \norm{A}_2.
\end{align*}
\end{proposition}

\subsubsection{Upper bounds on QCQP}\label{subsubsec:upperbounds}
This subsection will present several upper bounds on the optimal QCQP objective value. We will state a simple upper bound based on the spectral properties
of the QCQP matrix $A$. We will then bound the possible values of the SDP solutions and present two constant relative accuracy bounds.

\noindent\textbf{$L_2$ norm bound}
We will present an upper bound on the objective of (\ref{eq:qcqp}) based on the largest eigenvalue of the problem matrix $A$.
\begin{proposition}
The objective value of (\ref{eq:qcqp}) is upper bounded by $m \cdot \norm{A}_2$.
\end{proposition}
\begin{proof}
The problem (\ref{eq:qcqp}) remains the same if we add a redundant constraint $x^T x = m$ obtained by summing the constraints $\sum_{i= 1}^{m} \left(x^{(i)T}x^{(i)} - 1\right) = 0$. We then relax the problem by dropping the original constraints to get:
\begin{equation}\label{eq:evp}
\begin{aligned}
& \underset{x \in \RR^{N}}{\text{max}}
& & x^T A x\\
& \text{subject to}
& &x^{T} x = m.
\end{aligned}
\end{equation}
Since $\norm{A}_2 = \text{max}_{\norm{x}_2 = 1} x^T A x$ it follows that the optimal objective value of (\ref{eq:evp}) equals $m \cdot \norm{A}_2 $.
\qed
\end{proof}
\noindent\textbf{Bound on possible SDP objective values}
\begin{lemma}\label{eq:lemsdp}
Let $X^*$ be the solution to the problem (\ref{eq:sdp}) and let $\psi := \mathrm{trace}\left(A X^{*}\right)$. Then
$$m \leq \psi \leq m^2.$$
\end{lemma}
\begin{proof}
Express $X^*$ as:
$$X^* =  \underset{i=1,\ldots, n}{\sum} \lambda_i x_i x_i^T,$$
where each $x_i$ has unit length and $\lambda_1 \geq \ldots \geq \lambda_N \geq 0$.
The lower bound follows from the fact that $\psi$ upper bounds the optimal objective value of problem (\ref{eq:qcqp}) which is lower bounded by $m$. The lower bound corresponds to the case of zero sum of correlations.

To prove the upper bound first observe that the constraints in (\ref{eq:sdp}) imply that $\underset{i=1,\ldots, n}{\sum}\lambda_i = m$.
Let $y \in \RR^{N}$ and let $\norm{y}_2 = 1$. Let $z := \left(\norm{y^{(1)}}, \ldots, \norm{y^{(m)}}\right)^T.$ Observe that $\norm{z}_2 = 1$ and that $\norm{z z^T}_2 = 1$. Define $e \in \RR^m, e\left(i\right) = 1,  \forall i=1,\ldots,m$.
We will now bound $\norm{A}_2$:
\begin{align*}
  y^T A y &= \sum_{i, j = 1,\ldots,m} y^{(i)T} A^{(i,j)} y^{(j)}\\
& = \sum_{i, j = 1,\ldots,m} \norm{y^{(i)}} \norm{y^{(j)}} \frac{y^{(i)T}}{\norm{y^{(i)}}} A^{(i,j)} \frac{y^{(j)}}{\norm{y^{(j)}}} \leq \\
&\leq \sum_{i, j = 1,\ldots,m} \norm{y^{(i)}} \norm{y^{(j)}} \\
&= e^T (z z^T) e \leq \norm{e}\cdot \norm{z z^T}_2 \cdot \norm{e} = m.
\end{align*}
We used the fact that $\frac{y^{(i)T}}{\norm{y^{(i)}}} A^{(i,j)} \frac{y^{(j)}}{\norm{y^{(j)}}}$ is a correlation coefficient and thus bounded by $1$.
The upper bound follows:
$$\mathrm{trace}\left(A X^{*}\right) = \mathrm{trace}\left(A  \underset{i=1,\ldots, n}{\sum} \lambda_i x^{(i)} x^{(i)T}\right) =$$
$$=  \underset{i=1,\ldots, n}{\sum} \lambda_i x^{(i)T} A x^{(i)} \leq \underset{i=1,\ldots, n}{\sum} \lambda_i \cdot m = m^2.$$
\qed
\end{proof}

\noindent\textbf{Constant relative accuracy guarantee}
We now state a lower bound on the ratio between the objective values of the original and the relaxed problem that is independent on the problem dimension. The bound is based on the following result from \cite{Nesterov98globalquadratic}, stated with minor differences in notation. Let $\mathrm{Square}(\cdot)$ denote componentwise squaring: if $y = \mathrm{Square}(x)$ then $y(i) = x(i)^2$ and let $\mathrm{diag}(X)$ denote the vector corresponding to the diagonal of the matrix $X$.

\begin{theorem}\label{thm:nesterov}
Let $A \in \RR^{N \times N}$ be symmetric and let $\mathcal{F}$ be a set with the following properties:
\begin{itemize}
\item $\mathcal{F}$ is closed, convex and bounded.
\item There exists a strictly positive $v \in \mathcal{F}$.
\item $\mathcal{F} = \left\{ v \in K: B v = c  \right\},$ where $K$ is a convex closed pointed cone in $\RR^N$ with non-empty interior, $B \in \RR^{k\times N}$, $c \neq 0_k$ and $\left\{  v \in \mathrm{int}K : B v = c \right\} \neq \emptyset$.
\end{itemize}
Let
\begin{align*}
\phi^* &:= \mathrm{max}\left\{  x^T A x : \mathrm{square}\left(x\right) \in \mathcal{F}  \right\},\\
\phi_* &:= \mathrm{min}\left\{  x^T A x : \mathrm{square}\left(x\right) \in \mathcal{F}  \right\},\\
\psi^* &:= \mathrm{max}\left\{  \mathrm{trace}\left(A X\right): \mathrm{diag}\left(X\right) \in \mathcal{F}, X \in \sym_N^+  \right\},\\
\psi_* &:= \mathrm{min}\left\{  \mathrm{trace}\left(A X\right): \mathrm{diag}\left(X\right) \in \mathcal{F}, X \in \sym_N^+  \right\},\\
\psi\left(\alpha\right) &:= \alpha \psi^* + \left(1-\alpha\right)\psi_*.
\end{align*}
Then  $$\psi_* \leq \phi_* \leq \psi \left(1 - \frac{2}{\pi}\right) \leq \psi\left(\frac{2}{\pi}\right) \leq \phi^* \leq \psi^*.$$
\end{theorem}

\begin{theorem}
Let
$x^{*}$ be the solution to the problem (\ref{eq:qcqp2}) and
$X^*$ be the solution to the problem (\ref{eq:sdp}).
%$x_{*}$ be the solution to the problem (\ref{eq:qcqp2min}), and
%$X_*$ be the solution to the problem (\ref{eq:sdpmin}).
Let $b = \left(n_1,\ldots,n_m\right)$ denote the block structure where $\sum_i n_i = N$.
%Let $x^{(i)}_k$ denote the $i$-th block of the $k$-th eigenvector of $X^*$. The dimension of $x^{(k)}_i$ is $n_i$.
Let $\phi^*:= \mathrm{trace}\left(A \cdot x^{*} \cdot x^{*T}\right)$,
$\psi^* := \mathrm{trace}\left(A X^{*}\right)$.
%$\phi_*:= \mathrm{trace}(A \cdot x_{*} \cdot x_{*}^T)$
%and
%$\psi_* := \mathrm{trace}(A X_{*})$.

Then $$\frac{2}{\pi} \psi^* \leq \phi^* \leq \psi^*.$$
\end{theorem}

\begin{proof}
We first note that $\psi_* \geq 0$, since $A \in \sym_{+}^n$. This follows from the fact that $\mathrm{trace}\left(A X \right) \geq 0$ for any $X \in \sym_+^n$ (and thus for the minimizer $X_*$). The positiveness of the trace can be deduced from: $\mathrm{trace}\left(A X \right) = \mathrm{trace}\left(C_A C_A^T C_X C_X^T \right) = \mathrm{trace}\left(C_X^T C_A C_A^T C_X \right) = \norm{C_A^T C_X}_F^2 \geq 0 $, where $C_A$ and $C_X$ are Cholesky factors of matrices $A$ and $X$ respectively.

We now show that the problems (\ref{eq:qcqp2}) and (\ref{eq:sdp}) can be reformulated so that the theorem \ref{thm:nesterov} applies.

First we note that the feasible sets in (\ref{eq:qcqp2}) and
(\ref{eq:sdp}) are defined in terms of equalities. Without loss
of generality we can replace them with inequality constraints:
$x^{(i)T}x^{(i)} \leq 1$ in (\ref{eq:qcqp2}) and $\mathrm{trace}
\left(B_i X\right) \leq 1$ in (\ref{eq:sdp}). The feasible sets
defined by the inequalities are convex and bounded. Since the
objective functions in both problems are convex, it follows that
the optima lie on the border.
%\item Restate the constraints corresponding to (\ref{eq:qcqp2}) as: $\mathrm{Square}(x^{(i)}) \leq 1$ and   constraints in terms of squares of variables and matrix diagonals

Next, we add redundant constraints to the two problems respectively: $\mathrm{Square}\left(x^{(i)}\right) \geq 0, \forall i = 1,\ldots,m$  and $X\left(j,j\right) \geq 0, \forall j = 1,\ldots, N$.  %add nonnegativity constraints on the square to F to obtain a product of simplices

Define $\mathcal{F} = \left\{x \in \RR^N | x^{(i)} \in \Delta^{n_i -1} \right\}$, where $$\Delta^{k} = \left\{ x \in \RR^{k+1} | x\left(i\right) \geq 0, \forall i ~\mathrm{and}~ \sum_i x\left(i\right) = 1 \right\}.$$
$\mathcal{F}$ is a product of standard simplices: $\mathcal{F} = \prod_{i = 1}^m \Delta^{n_i -1}$. It follows that the set is closed, bounded and convex.
 $\mathcal{F}$ can be embedded in $\RR^{N+1}$ in order to obtain a conic formulation.
\begin{align*}
K = \{ t\cdot \left[1~ x^T\right]^T &| t \geq 0, x \in \mathcal{F} \} \\
B = \left[1~ 0_N^T\right]^T,\quad c = 1,& \quad \widetilde{\mathcal{F}} = K \cap \left\{x | B x = c \right\}.
\end{align*}

 Define $v = \left[v_1^T \ldots v_m^T\right]^T,$ where $v_i\left(j\right) = \frac{1}{n_i}$. %The vector lies in the interior of $\mathcal{F}$.
 The vector $\left[1~ v^{T}\right]^T$ is strictly positive and lies in $\mathrm{int}\left(K\right) \cap \left\{x \in \RR^{N+1} | B x = c \right\}$.
 Let $\widetilde{A} \in \RR^{N+1}$ be defined as $\widetilde{A}\left(1,i\right) = 0$, $\widetilde{A}\left(i,1\right) = 0, \forall i$ and $\widetilde{A}\left(i,j\right) = A\left(i-1,j-1\right), \forall i,j > 1$.

 The optimization problem (\ref{eq:qcqp2}) is equivalent(with
 the same optimal objective value) to:
\begin{equation*}%\label{eq:qcqp2}
%\tag{QCQP2}
\begin{aligned}
& \underset{x \in \RR^{N+1}}{\text{max}}
& &\mathrm{trace}\left(\widetilde{A} x x^T\right) \\
& \text{subject to}
& &\mathrm{Square}\left(x\right) \in \widetilde{F}\\
\end{aligned}
\end{equation*}
The optimization problem (\ref{eq:sdp}) is likewise equivalent to the problem:
\begin{equation*}
\begin{aligned}
& \underset{X \in \sym_{+}^{N+1}}{\text{max}}
& &\mathrm{trace}\left(\widetilde{A} X\right) \\
& \text{subject to}
& &\mathrm{diag}\left(X\right) \in \widetilde{F}\\
\end{aligned}
\end{equation*}

 %$$ \mathrm{find max }\left\{x^T   \right\}.$$ %to optimizing over $\left\{ x \in \RR^{N+1} | \mathrm{Square}(x) \in \widetilde{\mathcal{F}} \right\}$. Similarly, the feasible set in (\ref{eq:sdp}) can be reformulated as $\left\{X \in \sym_+^N | diag(X) \in \mathcal{F} \right\}$.
%\item Two equivalent formulations with the same objective values

%Let $\mathrm{Square}(\cdot)$ denote componentwise squaring: if $y = \mathrm{Square}(x)$ then $y(i) = x(i)^2$.
%Define $\psi(\alpha) := \alpha \psi^* + (1-\alpha)\psi_*.$
%The constraints in optimization problems (\ref{eq:qcqp2}) and (\ref{eq:qcqp2min}) represent linear equalities on squares of variables (thus the squared variables lie in a closed convex set). It is easy to see that assumption 2.1 and assumption 2.2 in \cite{Nesterov98globalquadratic} are satisfied. Furthermore the constraints in optimization problems (\ref{eq:sdp}) and (\ref{eq:sdpmin}) represent linear equalities on the vector corresponding to the diagonal of feasible $X \in S^n_+$. This means that our optimization problems are special cases of optimization problems considered in \cite{Nesterov98globalquadratic}.
% We can now apply Theorem 2.1 in \cite{Nesterov98globalquadratic} which states:
% $$\psi_* \leq \phi_* \leq \psi(1 - \frac{2}{\pi}) \leq \psi(\frac{2}{\pi}) \leq \phi^* \leq \psi^*.$$
 %POPRAVEK : $\psi(\alpha) \geq \alpha \psi_* je narobe
 %why is \psi_* \geq 0 ?

 Using the definition of $\psi\left(\alpha\right)$ and the fact that $\psi_* \geq 0$ it follows that $\psi\left(\alpha\right) \geq \alpha \psi^* ,\forall \alpha \geq 0$. Substituting $\alpha = \frac{2}{\pi}$ we get the desired result:
 $$\frac{2}{\pi} \psi^* \leq \phi^* \leq \psi^*.$$
 \qed
\end{proof}

Observe that the bound above relates the optimization problems (\ref{eq:qcqp}) and (\ref{eq:sdp}) and not (\ref{eq:qcqp0}) with its SDP relaxation. Let $\widetilde{\phi}$ denote the optimum value of the objective function in (\ref{eq:qcqp0}) and let $\widetilde{\psi}$ denote the optimum value of the objective function of the corresponding SDP relaxation. It is easy to see that $2 \cdot \widetilde{\phi} + m = \phi$ and $2 \cdot \widetilde{\psi} + m = \psi$, which is a consequence of transformations of the original problems to their equivalent symmetric positive-definite problems. The $\frac{2}{\pi}$ constant relative accuracy bound becomes a bit weaker in terms of the original problem and its relaxation. This fact is stated in the following corollary.
\begin{corollary}
$$\widetilde{\phi} \geq \frac{2}{\pi} \widetilde{\psi} - \frac{(1 - \frac{2}{\pi}) m}{2}.$$
\end{corollary}

\noindent\textbf{Improved bound on the relative accuracy}
We can exploit additional structure of the problem to obtain a slightly better bound. We use the same conventions as \cite{Nesterov98globalquadratic}.

Define
$$\omega\left(\beta\right) := \beta \arcsin\left(\beta\right) + \sqrt{1 - \beta^2}.$$
The function $\omega\left(\beta\right)$ is increasing and convex with $\omega\left(0\right) = 1$ and $\omega\left(1\right) = \frac{\pi}{2}$.

By theorem 3.1, item 1 in \cite{Nesterov98globalquadratic} we obtain the result:
$$ \max\left\{\frac{2}{\pi}\omega\left(\frac{m}{\psi^*}\right), \frac{m}{\psi^*} \right\}   \psi^* \leq \phi^* \leq \psi^*.$$

This results in a minor improvement of the default bound. For example when $m = 3$ and the fact that $\frac{m}{\psi^*} \geq \frac{1}{3}$ we obtain the following:
$$ \frac{2}{\pi} \psi^* \leq \frac{105}{100}  \cdot \frac{2}{\pi} \psi^* \leq \phi^* \leq \psi^*$$

\section{Sum of correlations extensions}\label{sec:sumcorextensions}
In this section we discuss two extensions of MCCA. By using kernel methods we show
how to find nonlinear dependencies in the data. We then present an extension of the method to finding more then one set of correlation vectors.
\subsection{Dual representation and kernels}\label{subsec:kernels}
%\label{sect:primal}
We return to the formulation (\ref{eq:qcqp0}):
 \begin{equation*}
\begin{aligned}
& \underset{w \in \RR^N}{\text{max}}
& & \sum_{i = 1}^m \sum_{j = i+1 }^m w^{(i)T} C^{(i,j)} w^{(j)}\\
& \text{subject to}
& &w^{(i)T} C^{(i,i)} w^{(i)} = 1, \quad\forall i = 1,\ldots, m,
\end{aligned}
\end{equation*}
  where $b = \left(n_1,\ldots,n_m\right)$ denotes the block structure and $ \sum_i n_i = N $. In the previous sections we focused on manipulating covariance matrices only and omitted details on their estimation based on finite samples. In this section we will use a formulation that explicitly presents the empirical estimates of covariances, which will enable us to apply kernel methods.
Let $\mathcal{X}$ be a random vector distributed over $\RR^N$ with $E\left(\mathcal{X}\right) = 0$. Let $X \in \RR^{N \times s}$ represent a sample of $s$ observations of $\mathcal{X}$, where each observation corresponds to a column vector. Empirical covariance of $\mathcal{X}$ based on the sample matrix $X$ is expressed as: $$ \overline{Cov\left(\mathcal{X}\right)} = \frac{1}{s - 1}X X^T.$$
In case the number of number of observations, $s$, is smaller than the total number of dimensions $N$, the covariance matrix $\overline{Cov\left(\mathcal{X}\right)}$ is singular. This is problematic both from a numerical
point of view and it leads to overfitting problems. These issues are addressed by using regularization techniques, typically a shrinkage estimator $\overline{Cov\left(\mathcal{X}\right)_{\kappa}}$ is defined as: $$ \overline{Cov\left(\mathcal{X}\right)_{\kappa}} = \left(1-\kappa\right) \frac{1}{s - 1}X X^T + \kappa  I_N,$$ where $\kappa \in \left[0,1\right]$.

Using the block structure $b$, (\ref{eq:qcqp05}) becomes:
 \begin{equation}\label{eq:regqcqp}
\begin{aligned}
& \underset{w \in \RR^N}{\text{max}}
& & \frac{1}{s -1} \sum_{i = 1}^m \sum_{j = i+1}^m w^{(i)T} X^{(i)}X^{(j)T} w^{(j)} \\
& \text{subject to}
& & w^{(i)T} \left(\frac{1- \kappa}{s - 1}X^{(i)} X^{(i)T} + \kappa  I_N\right) w^{(i)} = 1,\\&&& \quad\forall i = 1,\ldots, m.
\end{aligned}
\end{equation}

We will now express each component $w^{(i)}$ in terms the columns of $X^{(i)}$. Let $w$ have block structure $b_w = \left(n_1, \ldots, n_m\right)$ where $\sum_i n_i = N$, and let $y \in \RR^{m\cdot s}$ have block structure $b_y\left(i\right) = s, \forall i = 1,\ldots, m$.

\begin{equation}\label{eq:representer}
\begin{aligned}
w^{(i)} = \sum_{j = 1}^{s} y^{(i)}\left(j\right) X^{(i)}\left(:,j\right) = X^{(i)} y^{(i)},
\end{aligned}
\end{equation}
We refer to $y$ as dual variables.

%It remains to check that the formulations (\ref{eq:regqcqp}) and (\ref{eq:dualregqcqp}) are equivalent. We need to check that the optimal solution to (\ref{eq:regqcqp}) can be expressed by using dual variables \eqref{eq:representer}.
\begin{lemma}
Solutions to the problem \eqref{eq:regqcqp} can be expressed as \eqref{eq:representer}.
\end{lemma}
\begin{proof}
We will prove this by contradiction. Let $u$ be the optimal solution to \eqref{eq:regqcqp}. Assume that $u^{(1)}$ doesn't lie in the column space of $X^{(1)}$,
 $$u^{(1)} = z_{\bot} + X^{(1)} y^{(1)},$$
 where
$$z_{\bot} \neq 0_{n_1}\quad \text{and}\quad X^{(1)T}z_{\bot} = 0_s.$$ Then $\bar{u}$ defined as $\bar{u}^{(i)} = u^{(i)}, \forall i> 1$ and $\bar{u}^{(1)} = \frac{1}{\gamma}  X^{(1)} y^{(1)} ,$ where
\begin{equation*}
\gamma =\sqrt{ y^{(1)T} X^{(1)T} \left(\frac{1- \kappa}{s - 1}X^{(1)} X^{(1)T} + \kappa  I_N\right) X^{(1)} y^{(1)} }
\end{equation*}
strictly increases the objective function, which contradicts $u$ being optimal. Clearly $\bar{u}$ is a feasible solution. Positive definiteness of $\frac{1- \kappa}{s - 1}X^{(1)} X^{(1)T} + \kappa  I_N$ coupled with the fact that $z_{\bot}^T z_{\bot} > 0$ implies that $0 < \gamma < 1$. Assume without loss of generality that $\sum_{j = 2}^m \left(X^{(1)} y^{(1)}\right)^T X^{(1)}X^{(j)T} u^{(j)} > 0$ (The negative sum would lead to another contradiction by taking $\bar{u}^{(1)} = - u^{(1)}$. If the sum was zero, then any properly scaled (with proper sign) combination of the training data $X^{(1)}$ could be used in place of $u^{(1)}$). The following inequality completes the proof:
%Note that $\frac{1}{s -1}  u^{(i)T} X^{(i)}X^{(j)T} u^{(j)} = \frac{1}{s -1}  \bar{u}^{(i)T} X^{(i)}X^{(j)T} \bar{u}^{(j)}$ where $i > 1$ and $j > 1$ and $u^{(1)T} (\frac{1- \kappa}{s - 1}X^{(1)} X^{(1)T} + \kappa  I_N) u^{(1)} = \bar{u}^{(1)T} (\frac{1- \kappa}{s - 1}X^{(1)} X^{(1)T} + \kappa  I_N) \bar{u}^{(1)} = 1$.
\begin{align*}
  \frac{1}{s -1} & \sum_{j = 2}^m u^{(1)T} X^{(1)}X^{(j)T} u^{(j)} =\\
= \frac{1}{s -1} & \sum_{j = 2}^m \left(z_{\bot} + X^{(1)} y^{(1)}\right)^T X^{(1)}X^{(j)T} u^{(j)} = \\
= \frac{1}{s -1}  &\sum_{j = 2}^m \left(X^{(1)} y^{(1)}T\right) X^{(1)}X^{(j)T} u^{(j)} < \\
< \frac{1}{s -1}  &\sum_{j = 2}^m \frac{1}{\gamma}\left(X^{(1)} y^{(1)}\right)^T X^{(1)}X^{(j)T} u^{(j)}.
\end{align*}
\qed
\end{proof}

%\substack{j = 1\\ j\neq i}

Let $K_i = X^{(i)T} X^{(i)} \in \RR^{s \times s}$ denote the Gram matrix. We now state regularized covariance formulation (\ref{eq:regqcqp}) in terms of the dual variables:
 \begin{equation}\label{eq:dualregqcqp}
\begin{aligned}
& \underset{y \in \RR^{m\cdot s}}{\text{max}}
& & \frac{1}{s -1} \sum_{i = 1}^m \sum_{j = i+1}^m y^{(i)T} K_i K_j^T y^{(j)} \\% + \sum_{i = 1}^m y^{(i)T} (\frac{1 - \kappa}{s - 1}K_i  K_i^T + \kappa  K_i) y^{(i)}\\
& \text{subject to}
& & y^{(i)T} \left(\frac{1- \kappa}{s - 1}K_i K_i^T + \kappa  K_i\right) y^{(i)} = 1,\\
& &&\quad\forall i = 1,\ldots, m.
\end{aligned}
\end{equation}

The problem is reformulated in terms of Gram matrices based on the standard inner product. This formulation lends itself to using kernel methods (see \cite{shawe-taylor04kernel}) which
enable discovering nonlinear patterns in the data.

Typically the matrices $K_i$ are ill conditioned (even singular when the data is centered) and it is advantageous to constrain the magnitude of dual coefficients as well as the variance in the original problem. We address this by introducing a first order approximation to the dual regularized variance.  Let $$\widetilde{K_i} := \left(\sqrt{\frac{1-\kappa}{s - 1}}K_i + \frac{\kappa}{2} \sqrt{\frac{s-1}{1- \kappa}}I_s\right).$$ Then:
 $$ \overline{Cov\left(\mathcal{X}^{(i)}\right)_{\kappa}} =  \frac{1- \kappa}{s - 1}K_i K_i^T + \kappa  K_i \approx  \widetilde{K_i} \widetilde{K_i}^T.$$
The approximation that has two advantages: it is invertible and factorized, which we exploit in obtaining a convergent local method.
The final optimization is then expressed as:
 \begin{equation}\label{eq:approxdualqcqp}
\begin{aligned}
& \underset{y \in \RR^{m\cdot s}}{\text{max}}
& & \frac{1}{s -1} \sum_{i = 1}^m \sum_{j = i+1}^m y^{(i)T} K_i K_j^T y^{(j)}\\% + \sum_{i = 1}^m y^{(i)T} \widetilde{K_i} \widetilde{K_i}^T y^{(i)}\\
& \text{subject to}
& & y^{(i)T} \widetilde{K_i} \widetilde{K_i}^T y^{(i)} = 1, \quad\forall i = 1,\ldots, m.
\end{aligned}
\end{equation}
The problem can be interpreted as maximizing covariance while constraining variance and magnitude of dual coefficients. %Notice that the sum in the objective $\sum_{i = 1}^m y^{(i)T} \widetilde{K_i} \widetilde{K_i}^T y^{(i)}$ is constant and thus redundant. It is useful to keep it in order to obtain a reformulation of the problem where the local method provably converges.

%\subsection{Regularization}
%To prevent overfitting and make the problem numerically stable (as in CCA) we propose a regularization scheme. Let $\kappa \in [0,1]$ and let $\tilde{K}_i := (1-\kappa)K_i + \kappa I$, solve:
%\begin{equation}\label{equation:regdual}\max_{\beta_1, \ldots, \beta_m} \sum_{i < j} \beta_i' K_i K_j \beta_j,\end{equation} s.t. $$\beta_i'\tilde{K}_i \tilde{K}_i \beta_i = 1, \quad \forall i.$$
%The main benefit of this form of regularization is the apriori Cholesky-like decomposition (the factors are not triangular). This form of regularization is provably equivalent to regularization in \cite{FBMJ}.

\subsection{Computing several sets of canonical vectors}\label{subsec:severalCanonicalVectors}
Usually a one-dimensional representation does not sufficiently capture
all the information in the data and higher dimensional subspaces are
needed. After computing the first set of primal canonical vectors we proceed to computing the next set. The next set should be almost as highly correlated as the first one, but essentially ``different'' from the first one. We will achieve this by imposing additional constraints for every view, namely that all projection vectors in view $i$ are uncorrelated with respect to $\widetilde{K}_i^2$ (similar as in two view regularized kernel CCA, see \cite{FBMJ}).
\par
Let $Y = \left[y_1, \ldots, y_k\right] \in \RR^{m\cdots \times k}$ represent $k$ sets of canonical vectors, where
$$Y^{(\ell)T} \widetilde{K_{\ell}^2} Y^{(\ell)} = I_k  \forall \ell = 1,\ldots, m. $$
The equation above states that each canonical vector has unit regularized variance and that different canonical vectors corresponding to the same view are uncorrelated (orthogonal with respect to $\widetilde{K_i^2}$).

%%\delta_{ij} \forall \ell = 1,\ldots, m$$
%%where $$\delta_{ij}  = \left\{ \begin{array}{lll}
%%0 & {\rm for} ~i \neq j \\
%%1 & {\rm for} ~i = j   \end{array} \right.$$

%%_i = (y_i^1, \ldots, y_i^k)$ is the matrix
%%of $k$ uncorrelated vectors with respect to $\tilde{K}_i$, for every
%%view $i$. We are searching for the set of vectors $\beta_1^{k+1},
%%\ldots, \beta_m^{k+1}$ with unit regularized variance that maximize
%%the \textsc{sumcor} objective and are uncorrelated with the first $k$
%%solutions:
%%$${\beta_i^{k+1}}' \tilde{K}_i^2 \beta_i^{j}, \forall j < k+1, \forall i,$$
%%which can be written as:
%%$${\beta_i^{k+1}}'\tilde{K}_i^2 B_i = 0 , \forall i.$$

We will now extend the set of constraints in the optimization (\ref{eq:approxdualqcqp}) to enforce the orthogonality.
 \begin{equation}\label{eq:kdimapproxdualqcqp}
\begin{aligned}
& \underset{y \in \RR^{m\cdot s}}{\text{max}}
& & \frac{1}{s -1} \sum_{i = 1}^m \sum_{j = i+1}^m y^{(i)T} K_i K_j^T y^{(j)}\\% + \sum_{i = 1}^m y^{(i)T} \widetilde{K_i} \widetilde{K_i}^T y^{(i)}\\
& \text{subject to}
& & y^{(i)T} \widetilde{K_i} \widetilde{K_i}^T y^{(i)} = 1, \quad\forall i = 1,\ldots, m\\
& & & Y^{(i)T} \widetilde{K_i} \widetilde{K_i}^T y^{(i)} = 0_k, \quad\forall i = 1,\ldots, m.
\end{aligned}
\end{equation}
In order to use the Horst algorithm, we first use substitutions:
$$Z^{(i)} = \widetilde{K_i}Y^{(i)}, \quad z^{(i)} = \widetilde{K_i}y^{(i)}.$$

We then define operators $$P_i = I_s - \widetilde{K}_i Y^{(i)} Y^{(i)T} \widetilde{K}_i = I_s - Z^{(i)} Z^{(i)T},$$ which map to the space orthogonal to the columns of $\widetilde{K}_i Y^{(i)}$. Each $P_i$ is a projection operator: $P_i^2 = P_i,$ which follows directly from the identities above.
We restate the optimization problem in the new variables:
 \begin{equation}\label{eq:Zkdimapproxdualqcqp}
\begin{aligned}
& \underset{z \in \RR^{m\cdot s}}{\text{max}}
& & \frac{1}{s -1} \sum_{i = 1}^m \sum_{j = i+1}^m z^{(i)T} \widetilde{K_i}^{-T} K_i K_j^T \widetilde{K_j}^{-1} z^{(j)} \\%+ \sum_{i = 1}^m z^{(i)T} z^{(i)}\\
& \text{subject to}
& & z^{(i)T}  z^{(i)} = 1, \quad\forall i = 1,\ldots, m\\
& & & Z^{(i)T} z^{(i)} = 0_k, \quad\forall i = 1,\ldots, m.
\end{aligned}
\end{equation}
By using the projection operators, the optimization problem is equivalent to:
\begin{equation*}%\label{eq:projZkdimapproxdualqcqp}
\begin{aligned}
& \underset{z \in \RR^{m\cdot s}}{\text{max}}
& & \frac{1}{s -1} \sum_{i = 1}^m \sum_{j = i+1}^m z^{(i)T} P_i^T \widetilde{K_i}^{-T} K_i K_j^T \widetilde{K_j}^{-1} P_j z^{(j)}\\% + \sum_{i = 1}^m z^{(i)T}  z^{(i)}\\
& \text{s.t.}
& & z^{(i)T}  z^{(i)} = 1, \quad\forall i = 1,\ldots, m.
\end{aligned}
\end{equation*}
By multiplying the objective by $2$ (due to symmetries of $P_i, K_i$ and $\widetilde{K_i}$) and shifting the objective function by $\frac{m}{1 - \kappa}$, the problem is equivalent to:
\begin{equation}%\label{eq:projZkdimapproxdualqcqp}
\begin{aligned}
& \underset{z \in \RR^{m\cdot s}}{\text{max}}
& & \frac{1}{s -1} \sum_{i = 1}^m \sum_{\substack{j = 1\\ j\neq i}}^m z^{(i)T} P_i^T \widetilde{K_i}^{-T} K_i K_j^T \widetilde{K_j}^{-1} P_j z^{(j)}\\
&&& + \frac{1}{1-\kappa}\sum_{i = 1}^m z^{(i)T}  z^{(i)}\\
& \text{s.t.}
& & z^{(i)T}  z^{(i)} = 1, \quad\forall i = 1,\ldots, m.
\end{aligned}
\end{equation}
This last optimization can be reformulated as:
\begin{equation}\label{eq:projZkdimapproxdualqcqp}
\begin{aligned}
& \underset{z \in \RR^{m\cdot s}}{\text{max}}
& & z^T A z\\
& \text{subject to}
& & z^{(i)T}  z^{(i)} = 1, \quad\forall i = 1,\ldots, m,
\end{aligned}
\end{equation}
where $A \in \RR^{m\cdot s}$ with block structure $b\left(i\right) = s, \forall i = 1,\ldots, m$, defined by:
\begin{equation*}
 A^{(i,j)} = \left\{ \begin{array}{lll}
 \frac{1}{s -1} P_i^T \widetilde{K_i}^{-T} K_i K_j^T \widetilde{K_j}^{-1} P_j  & {\rm for} ~i \neq j\\
\frac{1}{1-\kappa } I_s & {\rm for} ~i = j \end{array}\right\}
\end{equation*}
\begin{lemma}
The block matrix $A$ defined above is positive semidefinite (i.e. $A \in \sym_+^{m\cdot s}$).
\end{lemma}
\begin{proof}
$A$ is symmetric, which follows from $P_i = P_i^T$ and $K_i = K_i^T$. Let $z \in \RR^{m\cdot s}$. We will show that $z^T A z > 0$.
Let $W =  \frac{1}{1- \kappa }\sum_{i = 1}^m z^{(i)T} P_i^T \widetilde{K_i}^{-T} \left( \kappa K_i + \frac{\kappa^2 \left(s-1\right)}{4\left(1-\kappa\right)}I_s   \right) \widetilde{K_i}^{-1} P_i z^{(i)}$. Note that $W \geq 0$ (each summand is positive-semidefinite) and $W > 0$ if $\exists i: P_i z^{(i)} = z^{(i)}$.
\begin{align*}
z^T A z &=  \frac{1}{s -1} \sum_{i = 1}^m \sum_{\substack{j = 1\\ j\neq i}}^m z^{(i)T} P_i^T \widetilde{K_i}^{-T} K_i K_j^T \widetilde{K_j}^{-1} P_j z^{(j)}\\& + \frac{1}{1- \kappa }\sum_{i = 1}^m z^{(i)T}  z^{(i)} \geq \\
&\geq \frac{1}{s -1} \sum_{i = 1}^m \sum_{\substack{j = 1\\ j\neq i}}^m z^{(i)T} P_i^T \widetilde{K_i}^{-T} K_i K_j^T \widetilde{K_j}^{-1} P_j z^{(j)}\\& + \frac{1}{1- \kappa }\sum_{i = 1}^m z^{(i)T} P_i^T P_i z^{(i)} = \\
 &= \frac{1}{s -1} \sum_{i = 1}^m \sum_{\substack{j = 1\\ j\neq i}}^m z^{(i)T} P_i^T \widetilde{K_i}^{-T} K_i K_j^T \widetilde{K_j}^{-1} P_j z^{(j)}\\& + \frac{1}{1- \kappa }\sum_{i = 1}^m z^{(i)T} P_i^T \widetilde{K_i}^{-T}  \widetilde{K_i}^T \widetilde{K_i} \widetilde{K_i}^{-1} P_i z^{(i)} = \\
&= \frac{1}{s -1} \sum_{i = 1}^m \sum_{\substack{j = 1\\ j\neq i}}^m z^{(i)T} P_i^T \widetilde{K_i}^{-T} K_i K_j^T \widetilde{K_j}^{-1} P_j z^{(j)}
\\& + \frac{1}{s-1}\sum_{i = 1}^m z^{(i)T} P_i^T \widetilde{K_i}^{-T}  K_i K_i^T \widetilde{K_i}^{-1} P_i z^{(i)}
 + W = \\ %+ \frac{1}{1- \kappa }&\sum_{i = 1}^m z^{(i)T} P_i^T \widetilde{K_i}^{-T} \left( \kappa K_i + \frac{\kappa^2 \left(s-1\right)}{4\left(1-\kappa\right)}I_s   \right) \widetilde{K_i}^{-1} P_i z^{(i)}
 &= z^T B B^T z + W \geq 0,
\end{align*}
where $B \in \RR^{m\cdot s \times s}$, defined by $B^{(i)} = \frac{1}{\sqrt{s-1}}(K_i \widetilde{K_i}^{-1}P_i)^T$, with corresponding row block structure $b\left(i\right) = s, \forall i = 1,\ldots, m$.
If $P_i z^{(i)} \neq z^{(i)}$ for some $i$, then the first inequality is strict ($\norm{P_i z^{(i)}} < \norm{z^{(i)}}$). Conversely, if $P_i z^{(i)} = z^{(i)}$ for all $i$, then $W > 0$, hence the last inequality is strict.\qed
\end{proof}

Since matrix $A$ has all the required properties for convergence, we can apply the Algorithm \ref{algorithm:horst}. Solutions to the problem \eqref{eq:kdimapproxdualqcqp} are obtained by back-substitution $y^{(i)} = \widetilde{K_i}^{-1} z^{(i)}$.

\subsection{Implementation}\label{subsec:implementation}
The algorithm involves matrix vector multiplications and inverted
matrix vector multiplications. If kernel matrices are products of
sparse matrices: $K_i = X^{(i)T} X^{(i)}$ with $X^{(i)}$ having $s n$ elements
where $s << n$, then kernel matrix vector multiplications cost $2 n s$
instead of $n^2$. We omit computing the full inverses and rather solve
the system $K_i x = y$ for x, every time $K_i^{-1} y$ is needed. Since
regularized kernels are symmetric and multiplying them with vectors is
fast (roughly four times slower as multiplying with original sparse matrices $X^{(i)}$), an iterative method like conjugate gradient (CG) is
suitable. Higher regularization parameters increase the condition
number of each $\tilde{K}_i$ which speeds up CG convergence.
\par
If we fix the number of iterations, $maxiter$, and number
of CG steps, $C$, the computational cost of computing a
$k$-dimensional representation is upper bounded by: $O\big(C \cdot
maxiter \cdot k^2 \cdot m \cdot n \cdot s \big),$ where $m$ is the
number of views, $n$ the number of observations and $s$ average number
of nonzero features of each observation.
Since the majority of computations is focused on sparse matrix-vector multiplications, the
algorithm can easily be implemented in parallel (sparse matrices are fixed and can be split into multiple
blocks).

So far we have assumed that the data is centered. Centering can efficiently be implemented on the
fly with no changes in asymptotic computational complexity, but we will omit the technical details due to space constraints.

\section{Experiments}\label{sec:experiments}

We evaluated the SDP approach on two scenarios: performance
analysis on synthetic data and performance on finding a common
representation of a cross-lingual collection of documents.
%, performance of the methods on the task of signal blind source separation and an exploratory analysis of sensor network data.

\subsection{Synthetic data}\label{subsec:syndata}
We generated several MCCA problem instances by varying the number
of views and number the of dimensions per view in order to
compare the performance of local search methods and the proposed
SDP relaxation. The main purpose of the experiments was to see
under which conditions and how likely do the global bounds
provide useful information.

 Let $m$ denote the number of views
(sets of variables) and $n_i$ denote the dimensionality of $i$-th
view and $N := \sum_i n_i$. In all cases we used the same number
of dimensions per view ($n_1 = n_2 = \cdots = n_m$). We used
three different methods to generate random correlation matrices.

The first method, the \textbf{random Gram matrices} (see
\cite{Holmes:1991:RCM:105724.105730}, \cite{Bendel_Mickey_78}) ,
generates the correlation matrices by sampling $N$ vectors $v_1,
\ldots, v_n$ for a $N$-dimensional multivariate Gaussian
distribution (centered at the origin, with an identity covariance
matrix), normalizing them and computing the correlation matrix $C
= \left[c_{i,j}\right]_{N \times N}$ as $c_{i,j} := v_i' \cdot v_j$.
The second method, the \textbf{random spectrum}, involves
sampling the eigenvalues $\lambda_1,\ldots,\lambda_N$ uniformly
from a simplex ($\sum_{i=1}^{N} \lambda_i = N$) and generating a
random correlation matrix with the prescribed eigenvalues (see
\cite{Bendel_Mickey_78}).
The final method, \textbf{random 1-dim structures}, involves
generating a correlation matrix that has an approximately (due to
noise) single-dimensional correlation structure. Here we
generated a random $m$ dimensional Gram matrix $B$, and inserted
it into a $N\times N$ identity matrix according to the block
structure to obtain a matrix $C_0$ (Set $C_0\left(i,j\right) = \delta\left(i,j\right)$,
where $\delta$ is the Kronecker delta. Then for $I,J = 1\ldots,m$
override the entries $C_0\left(1+ \sum_{i=1}^{I-1}n_i, 1+
\sum_{i=1}^{J-1}n_i\right) = B\left(I,J\right)$, where we used $1$-based
indexing). We then generated a random Gram matrix $D \in
\RR^{N\times N}$ and computed the final correlation matrix as $C
= \left(1- \epsilon\right)C_0 + \epsilon D$. In our experiments we set
$\epsilon = 0.001$. The purpose of using a random spectrum method
is that as the dimensionality increases, random vectors tend to
be orthogonal, hence the experiments based on random Gram
matrices might be less informative. As we will see later, the
local method suffers the most when all $n_i = 1$, which is an
instance of a BQO problem. By using the approximately
1-dimensional correlation matrix sampling we investigated if the
problem persists when $n_i > 1$.

In all cases we perform a final step that involves computing the
per-view Cholesky decompositions of variances and change of basis
as we showed when we arrived to QCQP reformulation in equation
(\ref{eq:qcqp}).

The experiments are based on varying the number of sets of
variables, $m$, and the dimensionality $n_i$. For each sampling
scenario and each choice of $m$ and $n_i$, we generated $100$
experiments, and computed $1000$ solutions based on Algorithm
\ref{algorithm:horst},  the SDP solution (and respective
global bounds), and looked at the frequencies of the following
events:
\begin{itemize}
\item a \textbf{duality gap} candidate detected (Tables
\ref{tb:rg}, \ref{tb:rs}, \ref{tb:r1d} (a)),
\item \textbf{local convergence} detected (Tables \ref{tb:rg}, \ref{tb:rs},
\ref{tb:r1d} (b))
\item when a local solution is worse
than the SDP-based \textbf{lower bound} (Tables \ref{tb:rg},
\ref{tb:rs}, \ref{tb:r1d} (c)).
\end{itemize}
The possibility of duality gap is
detected when the best local solution is lower than $1\%$ of the
SDP bound. In this case the event indicates only the possibility
of duality gap -- it might be the case that further local
algorithm restarts would close the gap. The local convergence
event is detected when the objective value of two local solutions
differs relatively by at least $10\%$ and absolutely at least by
$0.1$ (both criterions have to be satisfied
simultaneously). Finally, the event of local solution being
bellow the SDP lower bound means that it is bellow
$\frac{2}{\pi}$ of the optimal objective value of the SDP
relaxation.

We find that regardless of how we generated the data, the lower
SDP bound is useful only when $n_i = 1$ (Table \ref{tb:rg},
\ref{tb:rs}, \ref{tb:r1d} (c)) and the results are similar for
different choices of $m$. There are, however rare cases (less
than $0.1\%$) where the lower bound is useful for $n_i = 2$ and
even rarer (less than $0.01\%$) for $n_i = 3$.

The chance of local convergence increases as the number of views
$m$ increases which can be consistently observed for all choices
of $n_i$ and sampling strategies. Generating a problem where the
local algorithm likely converges to a local solution is less
likely as the dimensionality increases in the generic case
(Tables \ref{tb:rg}, \ref{tb:rs}).
%
%
%\textcolor{red}{Nasledni odstavek ne razumem}
%\textcolor{red}{Sem popravil. A je OK? Jan}
In the case of noisy embeddings of 1-dimensional correlation structures the dependence on $n_i$
behaves differently: the local convergence (see Table \ref{tb:r1d_lc}) for the case $\left(m=5, n_i=3\right)$ is more likely
than for the case $\left(m=5, n_i =2\right)$, which is a curiosity (in the general case, increasing $n_i$ reduces that chance of local convergence, see Table \ref{tb:rs_lc}, Table \ref{tb:rg_lc}).

The relationship between $m$ and $n_i$ and the possibility of
duality gap behaves similarly as the local convergence -
increasing $m$ increases it and increasing $n_i$ decreases it (Table \ref{tb:rg_dg}, Table \ref{tb:rs_dg}),
except in the case of noisy 1-dim correlation structures, where
we observe the same anomaly when $n_i = 2$ (Table \ref{tb:r1d_dg}).

To summarize, we illustrated the influence of $m$ and $n_i$ on
the performance of the Algorithm \ref{algorithm:horst} and
demonstrated that there exist sets of problems with nonzero
measure where the SDP bounds give useful information.
\begin{table}
\begin{center}
\caption{\label{tb:rg} Random Gram matrix sampling}
\begin{subtable}{.5\textwidth}
\caption{\label{tb:rg_dg} Possible duality gap}
\centering\begin{tabular}{|l|c|c|c|}
\hline
&\textbf{$n_i$ = 3}&\textbf{$n_i$ = 2}&\textbf{$n_i$ = 1}\\\hline
\textbf{$m$ = 5}&0\%&5\%&17\%\\\hline
\textbf{$m$ = 3}&0\%&0\%&9\%\\\hline
\end{tabular}

\end{subtable}\vspace{0.1cm}
\begin{subtable}{.5\textwidth}
\caption{\label{tb:rg_lc} Local convergence}
\centering\begin{tabular}{|l|c|c|c|}
\hline
&\textbf{$n_i$ = 3}&\textbf{$n_i$ = 2}&\textbf{$n_i$ = 1}\\\hline
\textbf{$m$ = 5}&1\%&5\%&48\%\\\hline
\textbf{$m$ = 3}&0\%&1\%&26\%\\\hline
\end{tabular}

\end{subtable}\vspace{0.1cm}
\begin{subtable}{.5\textwidth}
\caption{\label{tb:rg_lb} Local solution below lower SDP bound}
\centering\begin{tabular}{|l|c|c|c|}
\hline
&\textbf{$n_i$ = 3}&\textbf{$n_i$ = 2}&\textbf{$n_i$ = 1}\\\hline
\textbf{$m$ = 5}&0\%&0\%&14\%\\\hline
\textbf{$m$ = 3}&0\%&0\%&12\%\\\hline
\end{tabular}

\end{subtable}%
\end{center}
\end{table}

\begin{table}
\begin{center}
\caption{\label{tb:rs} Random spectrum sampling}
\begin{subtable}{.5\textwidth}
\caption{\label{tb:rs_dg} Possible duality gap}
\centering\begin{tabular}{|l|c|c|c|}
\hline
&\textbf{$n_i$ = 3}&\textbf{$n_i$ = 2}&\textbf{$n_i$ = 1}\\\hline
\textbf{$m$ = 5}&0\%&5\%&36\%\\\hline
\textbf{$m$ = 3}&0\%&1\%&20\%\\\hline
\end{tabular}

\end{subtable}\vspace{0.1cm}
\begin{subtable}{.5\textwidth}
\caption{\label{tb:rs_lc} Local convergence}
\centering\begin{tabular}{|l|c|c|c|}
\hline
&\textbf{$n_i$ = 3}&\textbf{$n_i$ = 2}&\textbf{$n_i$ = 1}\\\hline
\textbf{$m$ = 5}&1\%&3\%&50\%\\\hline
\textbf{$m$ = 3}&0\%&0\%&31\%\\\hline
\end{tabular}

\end{subtable}\vspace{0.1cm}
\begin{subtable}{.5\textwidth}
\caption{\label{tb:rs_lb} Local solution below lower SDP bound}
\centering\begin{tabular}{|l|c|c|c|}
\hline
&\textbf{$n_i$ = 3}&\textbf{$n_i$ = 2}&\textbf{$n_i$ = 1}\\\hline
\textbf{$m$ = 5}&0\%&0\%&15\%\\\hline
\textbf{$m$ = 3}&0\%&0\%&16\%\\\hline
\end{tabular}

\end{subtable}%
\end{center}
\end{table}
\begin{table}[t]
\begin{center}
\caption{\label{tb:r1d} Random 1-dim structure sampling}
\begin{subtable}{.5\textwidth}
\caption{\label{tb:r1d_dg} Possible duality gap}
\centering\begin{tabular}{|l|c|c|c|}
\hline
&\textbf{$n_i$ = 3}&\textbf{$n_i$ = 2}&\textbf{$n_i$ = 1}\\\hline
\textbf{$m$ = 5}&24\%&16\%&23\%\\\hline
\textbf{$m$ = 3}&7\%&4\%&7\%\\\hline
\end{tabular}

\end{subtable}\vspace{0.1cm}
\begin{subtable}{.5\textwidth}
\caption{\label{tb:r1d_lc} Local convergence}
\centering\begin{tabular}{|l|c|c|c|}
\hline
&\textbf{$n_i$ = 3}&\textbf{$n_i$ = 2}&\textbf{$n_i$ = 1}\\\hline
\textbf{$m$ = 5}&9\%&6\%&51\%\\\hline
\textbf{$m$ = 3}&0\%&0\%&31\%\\\hline
\end{tabular}

\end{subtable}\vspace{0.1cm}
\begin{subtable}{.5\textwidth}
\caption{\label{tb:r1d_lb} Local solution below lower SDP bound}
\centering\begin{tabular}{|l|c|c|c|}
\hline
&\textbf{$n_i$ = 3}&\textbf{$n_i$ = 2}&\textbf{$n_i$ = 1}\\\hline
\textbf{$m$ = 5}&0\%&0\%&13\%\\\hline
\textbf{$m$ = 3}&0\%&0\%&15\%\\\hline
\end{tabular}

\end{subtable}%
\end{center}
\end{table}
\subsection{Multilingual document collection}\label{subsec:documents}
%\textbf{Motivation}

Applications of canonical correlation analysis on collections of
documents have been demonstrated in dimensionality reduction,
cross-lingual document retrieval and classification
\cite{ccatext} \cite{ccatextdva}, extracting multilingual topics
from text \cite{mcca}, detecting bias in news
\cite{ccanewsbias}. In this section we explore the behavior of
Algorithm \ref{algorithm:horst} with respect to the global
bounds. We will start by describing the data and then describe a method to
reduce the dimensionality of the data in order to apply the SDP
bounds.

\noindent\textbf{Data set and preprocessing}
Experiments were conducted on a subset of EuroParl, Release v3,
\cite{europarl}, a multilingual parallel corpus, where our subset
includes Danish, German, English, Spanish, Italian, Dutch,
Portuguese and Swedish language. We first removed all documents
that had one translation or more missing. Documents (each
document is a day of sessions of the parliament) were then
arranged alphabetically and split into smaller documents, such that
each speaker intervention represented a separate document. We
removed trivial entries (missing translation) and after that
removed all documents that were not present in all eight
languages.

Thus we ended up with $12,000$ documents per
language. They roughly correspond to all talks between 2.25.1999
and 3.25.1999. We then computed the bag of words (vector space)
\cite{Salton88term-weightingapproaches} model for each language,
where we kept all uni-grams, bi-grams and tri-grams that occurred
more than thirty times. For example: "Mr", "President" and
"Mr\_President" all occurred more than thirty times in the
English part of the corpus and they each represent a dimension in
the vector space. This resulted in feature spaces with
dimensionality ranging from $50,000$ (English) to $150,000$
(German). Finally we computed the tf-idf weighting and normalized
every document for each language. We described how we obtained
corpus matrices $X^{(i)}$ for each language, where all the
matrices have $12,000$ columns and the columns are aligned
($X^{(i)}\left(:,\ell\right)$ and $X^{(j)}\left(:,\ell\right)$ are a translation of
each other). In section \ref{sec:sumcorextensions} we showed how to
derive the QCQP problem, given a set of input matrices $X^{(i)}$.

\noindent\textbf{Random projections and multivariate regression}
%experiments: languagesSDP_main
Applying the relaxation techniques to covariance matrices arising
from text presents a scalability problem, since both the number
of features (words in vocabulary) and number of documents can be
large. Typical SDP solvers can find solutions to relaxed forms of
QCQPs with up to a few thousand original variables. We now
propose a method to address this issue. The main goal is to
reduce the dimensionality of the feature vectors which results in
tractable SDP problem dimensions.
One way to analyze a monolingual document collection is to
perform singular value decomposition on the corpus matrix, a
technique referred to as Latent Semantic Indexing
(LSI)\cite{lsi}. A set of largest singular vectors can be used as
a new document basis for dimensionality reduction. Expressing the
documents with the basis of $k$ largest singular vectors is
optimal with respect to Frobenious norm reconstruction error. If
computing the basis is too expensive, one can generate a random
larger set of basis vectors that achieve similar reconstruction
errors, a technique referred to as random projections. Although
the random projection basis is not informative in the sense that
LSI basis is (topics extracted by LSI reflect which topics are
relevant for the corpus, as opposed to random topics), they can
both achieve comparable compression qualities.

A variant of LSI for multiple languages, Cross-Lingual LSI
(CL-LSI)\cite{ cl_lsi}, first joins all feature spaces thus
obtaining a single multilingual corpus matrix (single language
corpus matrices are stacked together). CL-LSI then proceeds as
standard LSI by computing the singular value decomposition of the
multilingual corpus matrix. Applying random projections instead
of SVD does not work directly; random multilingual topic vectors
destroy cross lingual information: a fact which can be observed
experimentally.

Our approach is based on the following idea. Generate a set of random vectors for one language and use Canonical Correlation Analysis Regression (CCAR)\cite{ccar} (a method similar to ridge regression) to find their representatives in the other languages. Repeat the procedure for each of the remaining languages to prevent bias to a single language. We hypothesize that restricting our search in the spaces spanned by the constructed bases still leads to good solutions. The procedure is detailed in Algorithm \ref{algorithm:rpgen}.

Let $m$ be the number of vector spaces corresponding to different
languages and $n_i$ the dimensionality of the $i-th$ vector
space. Let $X^{(i)} \in \RR^{n_i \times N}$ represent the aligned
document matrix for the $i$-th language.

\begin{algorithm}
\caption{Random projections basis generation}
\label{algorithm:rpgen}
{\bf Input:} matrices $X^{(1)},\ldots X^{(m)}$, $\gamma$ - the regularization coefficient, $k$ - the number of projections per block
\begin{algorithmic}
\FOR{$i = 1$ to $m$}
\STATE $P_{(i,i)} :=$ random $n_i \times k$ matrix where each element is sampled $i.i.d.$ from standard normal distribution.
\STATE Re-scale each column of $P_{(i,i)}$ so that its norm is equal to $\sqrt{\frac{n_i}{k}}$.
\FOR{$j = 1$ to $m$}
\IF {$j = i$}
 \STATE continue
\ENDIF
\STATE  $\alpha_{(i,j)} :=  \left(\left(1-\gamma\right) X^{(j)} X^{(j)T} + \gamma  I_j \right)^{-1}$
\STATE  $P_{(i,j)} :=  \alpha_{(i,j)} X^{(j)} X^{(i)T}  P_{(i,i)},$ where $I_j$ is the $n_j \times n_j$ identity matrix.
%\STATE  $P_{(i,j)} :=   \left(\left(1-\gamma\right) X^{(j)} X^{(j)T} + \gamma  I_j \right)^{-1} X^{(j)} X^{(i)T}  P_{(i,i)},$ where $I_j$ is the $n_j \times n_j$ identity matrix.
\ENDFOR
\ENDFOR
\\
\end{algorithmic}
{\bf Output:} matrices $P_{(i,j)} \;\text{for}\; i,j = 1,\ldots,m$
\end{algorithm}

The matrices $P_{(i,1)}, \ldots, P_{(i,m)}$ form the bases of
vector spaces corresponding to $X^{(1)},\ldots, X^{(m)}$. Let
$P_i := \left[P_{(1,i)}, \ldots, P_{(m,i)}\right]$ denote the full basis for
the $i$-th language.
We now experimentally address two questions: does the restricted
space enable us to find stable patterns and what do the SDP
bounds tell us.

%languagesSDP
%languagesSDP_main
%languageSDP339256474618

%languagesSDP_main(5000, 5, 50, 1000);
%      regprimalS: [0.0100 0.1000 0.5000 0.9000 0.9900]
%     ranprojregS: [0.1000 0.5000 0.9000 0.9900]
%            nexp: 10
%    ntrainPrimal: 5000
%      ntrainDual: 30
%           nview: 5
%        nranproj: 50
%        testsize: 1000
%      resultName: 'languageSDP339256474618'

\noindent\textbf{Experimental protocol}
 The experiments were conducted on the set of five EuroParl
 languages: English, Spanish, German, Italian and Dutch. We set
 $k = 10$ which corresponds to $n_i = 50$ dimensions per view, so
 the QCQP matrix will be of size $250 \times 250$. We randomly
 select $5000$ training documents and $1000$ test documents.  For
 a range of random projection regularization parameters $\gamma$,
 we compute the mappings $P_i$ (based on the train set) and
 reduce the dimensionality of the train and test sets. Then, for
 a range of QCQP regularization parameters $\kappa$, we set up the
 QCQP problem, compute $1000$ local solutions (by Horst
 algorithm) and solve the SDP relaxation. The whole procedure is
 repeated $10$ times.

For each $(\gamma, \kappa)$ pair we measured the sum of
correlations on the test and train sets. In Table
\ref{tb:textTrainTestSumcor} we report sums of correlations
averaged over $10$ experimental trials. The maximal possible sum
of correlations for five datasets equals to $\binom{5}{2} = 10$.
We observe that regularizing the whole optimization problem is
not as vital as regularizing the construction of random
projection vectors. This is to be expected since finding the
random projection vectors involves a regression in a high
dimensional space as opposed to solving a lower dimensional
QCQP. Selecting $\gamma = 0.1$ leads to perfectly correlated
solutions on the training set for all $\kappa$. This turns out to
be over-fitted when we evaluate the sum of correlations on the
test set. Note that higher $\kappa$ values in this case improve the
performance on the test set but only up to a certain level below
$7.5$. As we increase $\gamma$ to $0.5$, we see a reduction in
overfitting and $\gamma = 0.9$ results in comparable performance
on the test and train sets (the patterns are stable). We have
demonstrated a technique to reduce the dimensionality of the
original QCQP problem which still admits finding stable
solutions. The reduced dimensionality enables us to investigate
the behavior of the SDP relaxation.

For the SDP bounds we observed behavior that was similar to the high-dimensional synthetic (generic) case. That is
we found that the potential duality gap was very small and that the SDP and the Horst algorithm yielded the same
result. For this reason we omit the SDP results from Table \ref{tb:textTrainTestSumcor}.

\begin{table}[tbp]
\begin{center}
\caption{\label{tb:textTrainTestSumcor} Train and test sum of correlation}
\begin{subtable}{.5\textwidth}
\caption{\label{tb:trainText} Train set sum of correlations}
\centering\begin{tabular}{|l|c|c|c|c|}
\hline
&\textbf{$\gamma =$0.1}&\textbf{$\gamma =$0.5}&\textbf{$\gamma =$0.9}&\textbf{$\gamma =$0.99}\\\hline
\textbf{$\kappa =$0.01}&10.0&9.8&9.8&9.8\\\hline
\textbf{$\kappa =$0.1}&10.0&9.8&9.8&9.8\\\hline
\textbf{$\kappa =$0.5}&10.0&9.8&9.8&9.8\\\hline
\textbf{$\kappa =$0.9}&10.0&9.8&9.8&9.8\\\hline
\textbf{$\kappa =$0.99}&10.0&9.8&9.7&9.8\\\hline
\end{tabular}

\end{subtable}\vspace{0.1cm}
\begin{subtable}{.5\textwidth}
\caption{\label{tb:testText} Test set sum of correlations}
\centering\begin{tabular}{|l|c|c|c|c|}
\hline
&\textbf{$\gamma =$0.1}&\textbf{$\gamma =$0.5}&\textbf{$\gamma =$0.9}&\textbf{$\gamma =$0.99}\\\hline
\textbf{$\kappa =$0.01}&5.8&8.6&9.6&9.8\\\hline
\textbf{$\kappa =$0.1}&6.2&8.6&9.6&9.8\\\hline
\textbf{$\kappa =$0.5}&7.0&8.6&9.6&9.8\\\hline
\textbf{$\kappa =$0.9}&7.4&8.8&9.6&9.8\\\hline
\textbf{$\kappa =$0.99}&7.4&8.8&9.6&9.8\\\hline
\end{tabular}

\end{subtable}
\end{center}
\end{table}

\section{Discussion}\label{sec:discussion}

In the paper we studied a generalization of CCA to more than two
sets of variables. We showed that the complexity of the problem
is NP-hard and described a locally convergent method as well as
presented how to generalize the method to the nonlinear case with
several canonical variates.  Experimentally, we observe that the
performance of the local method (with linear convergence) is
generally good, although we identified problem settings where the
local method can be far from global optimality. We presented a
SDP relaxation of the problem, which can be used to obtain new
local solutions and to provide certificates of optimality. The
usefulness of the bounds was tested on synthetic problem
instances and a problems related to cross-lingual text
mining. The high dimensional nature of documents and the size of
the document collections result in untractable memory
requirements. We solved the issue by proposing a preprocessing
step based on random projections.

Future work includes analyzing the complexity of the other
generalizations proposed in \cite{Kettenring}. We found that
noisy 1-dimensional embeddings present difficulties for the local
approach as opposed to generic problem structures. A natural
question is, are there other problem structures that result in
suboptimal behavior of the local approach.  We presented result
based on textual data, however, this setting appears in many
setting and we plan to extend the list of applications to include
other modalities, such as images, sensor streams and graphs
(social media analysis).

\bibliographystyle{plain}\bibliography{qcqp_sdp}

\appendix

\section{Notation}\label{sec:notation}

This section reviews the notation used throughout the paper.
\begin{itemize}
\item Column vectors are denoted by lowercase
letters, e.g. $x$.
\item Matrices are denoted by uppercase letters,
e.g. $X$.
\item Constants will be denoted by letters of Greek alphabet,
e.g. $\alpha$.
\item Row vectors and transposed matrices are denoted
by $x^T$ and $X^T$ respectively.
\item Subscripts are used to
enumerate vectors or matrices, e.g. $x_1, x_2$, $X_1$, except in the
special case of the identity matrix, $I_n$ and the zero matrix $0_{k,l}$.
In these cases, the subscripts
denote row and column dimensions.
\item Parentheses next to vectors or
matrices are used to denote specific elements: $x(i)$ denotes the
$i$-th element of vector $x$ and $X(i,j)$ denotes the element in
the $i$-th row and $j$-th column of matrix $X$.
\item Notation $X(:,j)$
denotes the $j$-th column of $X$ and $X(i,:)$ denotes the $i$-th
row (This is MATLAB notation and simliar to the notation used
in~\cite{golub}).
\item Parenthesis are also used to explicitly denote the components of
a row vector: $$x = \left(x(1), x(2), \ldots, x(N)\right),$$ where
$x$ is an $N$-dimensional vector.
\item Let $\RR^n$ denote the
$n$-dimensional real vector space and $\RR^{n\times m}$ denote
the $(n \cdot m)$-dimensional vector space used when specifying
matrix dimensions and let $\NN$ denote the natural numbers.
\item Let $\sym_n^{+}$ denote the space of symmetric positive definite $n$-by-$n$ matrices.
\item Let double-struck capital letters denote vector spaces,
e.g. $\mathbb{X}$.
\item Horizontally concatenating two matrices with
the same number of rows, $A$ and $B$, is denoted by $[A~B]$,
e.g. stacking two column vectors $x_1$ and $x_2$ vertically is
denoted by $[x_1^T x_2^T]^T$.
\item Superscripted indices in
parenthesis denote sub-blocks of vectors or matrices
corresponding to a vector encoding the block structure: $b =
(n_1, \ldots, n_m)$ where $n_i \in \NN$ for $i = 1,\ldots,m$. We
use $x^{(i)}$ to denote the $i$-th sub-column of the vector $x$,
which by using the block structure $b$ and one-based indexing
corresponds to $$x^{(i)} := \left(x\left(\sum_{j=1}^{i-1}b(j) +
1\right), \ldots,
x\left(\sum_{j=1}^{i}b(j)\right)\right).$$
\item Matrix block notation
$A^{(i,j)}$ represents $i$-th row block and $j$-th column block
with respect to the block structure $b$. Using a single index in
matrix block notation, $A^{(i)}$, denotes a block row (column
dimension of $A^{(i)}$ equals the column dimension of $A$).
\item Let $\mathrm{Square}(\cdot)$ denote componentwise squaring: if $y = \mathrm{Square}(x)$ then $y(i) = x(i)^2$ and let $\mathrm{diag}(X)$ denote the vector corresponding to the diagonal of the matrix $X$.
\end{itemize}

%Random column vectors are denoted by calligraphic uppercase letters, for example $\mathcal{X}$ and superscript and subscripts are used in the same way as with column vectors.
%Expected value of a random vector is denoted as $E(\mathcal{X})$ and covariance between two random vectors $\mathcal{X}$ and $\mathcal{Y}$ is denoted by $Cov(\mathcal{X}, \mathcal{Y}) = E \left( \left( \mathcal{X} - E\left(\mathcal{X}\right) \right)\left(\mathcal{Y} - E\left(\mathcal{Y}\right) \right)^T \right)$.
%Sample, aligned sample
%Empirical expected value and empirical covariance estimated on a sample are denoted by $\overline{E(\cdot)}$ and $\overline{Cov(\cdot,\cdot)}$.

\end{document}